\documentclass{article}

 \usepackage[preprint]{neurips_2026}

\usepackage{microtype}
\usepackage{graphicx}
\usepackage{subcaption}
\usepackage{booktabs} 

\usepackage{hyperref}

\usepackage[dvipsnames]{xcolor}
\usepackage[normalem]{ulem}
\usepackage{tikz}
\usetikzlibrary{arrows}

\usepackage{bm}
\usepackage{multirow}
\usepackage{booktabs}
\usepackage{graphicx}  

\usepackage{enumitem}

\usepackage[utf8]{inputenc} 
\usepackage[T1]{fontenc}    
\usepackage{hyperref}       
\usepackage{url}            
\usepackage{booktabs}       
\usepackage{amsfonts}       
\usepackage{nicefrac}       
\usepackage{microtype}      
\usepackage{xcolor}         

\usepackage{amsmath}
\usepackage{amssymb}
\usepackage{mathtools}
\usepackage{amsthm}

\usepackage[capitalize,noabbrev]{cleveref}

\theoremstyle{plain}
\newtheorem{theorem}{Theorem}[section]
\newtheorem{proposition}[theorem]{Proposition}
\newtheorem{lemma}[theorem]{Lemma}
\newtheorem{corollary}[theorem]{Corollary}
\theoremstyle{definition}
\newtheorem{definition}[theorem]{Definition}

\newtheorem{example}[theorem]{Example}
\theoremstyle{remark}

\usepackage{wrapfig}

\usepackage[textsize=tiny]{todonotes}

\title{OrthoReg: Orthogonal Regularization for Hybrid Symbolic-Neural Dynamical Systems}

\author{%
  Till Richter \\
  Technical University of Munich \\
  Helmholtz Munich \\
  \texttt{till.richter@helmholtz-munich.de} \\
  \And
  Niki Kilbertus \\
  Technical University of Munich \\
  Helmholtz Munich \\
  \texttt{niki.kilbertus@helmholtz-munich.de} \\
}

\begin{document}

\maketitle


\begin{abstract}
Dynamical systems are fundamental to modeling the natural world, yet modeling them involves a persistent trade-off: 
manually prescribed mechanistic models are interpretable by design but often overly simplistic and misspecified; 
in contrast, flexible data-driven neural methods lack physical insight. 
Hybrid modeling aims for the best of both worlds by combining a prescribed or symbolic, physics-based component with a flexible neural network. 
A critical challenge, however, is that the neural component may relearn mechanistic parts, yielding redundant and uninterpretable models, especially when the symbolic structure itself is discovered from data. 
Existing methods based on standard $L^2$ regularization rely on a projection argument that breaks when the symbolic component is learned through sparse discovery, allowing the neural augmentation to overlap with symbolic structure.
We introduce \textbf{OrthoReg} (Orthogonal Regularization), which directly penalizes overlap between the symbolic and neural components, preventing symbolic structure from being absorbed by the neural residual. 
This yields a complementary decomposition: 
the symbolic part captures what the library can express, and the neural part captures what remains. 
On benchmark dynamical systems with partial library mismatch, OrthoReg improves symbolic recovery and out-of-distribution behavior.

\end{abstract}

\section{Introduction}
Dynamical systems modeling has long been a cornerstone across the sciences, especially for the natural and life sciences.
Applications range from healthcare data~\citep{10.5555/3157382.3157490, hess2024bayesian, Seedat2022Continuous, richter2026generative}, climate modeling~\citep{10.1145/3485128,eyring2024pushing}, and power systems~\citep{Toubeau2018Deep}, just to name a few. 
However, it faces a fundamental trade-off: symbolic models, traditionally specified by hand, provide interpretability by design, but typically cannot capture complex unknown phenomena; 
flexible neural networks instead excel at fitting data from dynamical systems~\citep{chen2018neural} but lack physical insight.

\noindent\makebox[\linewidth][c]{%
\begin{tikzpicture}[font=\scriptsize]
    \shade[left color=BrickRed!25, right color=JungleGreen!25, rounded corners=0.55ex] (0pt, -0.55ex) rectangle (0.95\linewidth, 0.55ex);
    \draw[rounded corners=0.55ex, draw=black!45, line width=0.25pt] (0pt, -0.55ex) rectangle (0.95\linewidth, 0.55ex);

    \fill[BrickRed] (0.55ex, 0pt) circle (0.35ex);
    \fill[black!60] (0.475\linewidth, 0pt) circle (0.35ex);
    \draw[-{latex}, thick, black!60]  (0.475\linewidth, 0pt) - ++(1.3,0);
    \draw[-{latex}, thick, black!60]  (0.475\linewidth, 0pt) - ++(-1.3,0);
    \fill[JungleGreen] (\dimexpr0.95\linewidth-0.55ex\relax, 0pt) circle (0.35ex);

    \node[anchor=south west, font=\scriptsize\bfseries, text=BrickRed] at (0.55ex, 0.7ex) {mechanistic};
    \node[anchor=south, font=\scriptsize\bfseries] (H) at (0.475\linewidth, 0.7ex) {hybrid};
    \node[anchor=south east, font=\scriptsize\bfseries, text=JungleGreen] at (\dimexpr0.95\linewidth-0.55ex\relax, 0.7ex) {data-driven};

    \node[anchor=north west, font=\tiny, align=left, text=BrickRed, text width=0.46\linewidth] at (0pt, -0.9ex) {fully prescribed, interpretable by design,\\ no data required, usually simple,\\ assumed to be the correct mechanistic system description};
    \node[anchor=north east, font=\tiny, align=right, text=JungleGreen, text width=0.46\linewidth] at (0.95\linewidth, -0.9ex) {fully data driven,\\complex black-box model,\\ approximate system behavior within data support};
\end{tikzpicture}%
}

Hybrid modeling becomes essential when domain knowledge is partial: 
epidemiological models may capture core transmission dynamics but miss behavioral feedback mechanisms that modulate contact rates; 
mechanical systems follow known laws of motion but exhibit complex friction and damping effects not easily expressible in closed form; 
climate models encode fundamental physics but require data-driven corrections for sub-grid processes. 
In these scenarios, a purely symbolic approach underperforms due to missing phenomena, while purely neural models sacrifice the interpretability and physical consistency that domain experts require for scientific insight and decision-making.
Hybrid modeling approaches~\citep{rackauckas2020universal,yin2021augmenting,zou2024hybrid} combine physical priors (predetermined symbolic or parametrized expressions) with learned neural corrections expected to capture phenomena that are unknown or too complex to model directly.
They promise the best of both worlds, but still require substantial prior knowledge in crafting the mechanistic part.
This work focuses on sparse library-based discovery of the symbolic component, as in SINDy-style methods~\citep{brunton2016discovering}. 
In this regime, the symbolic component is fitted with a continuous sparsity penalty (e.g.\ $L^1$ regularisation), where coefficient shrinkage leaves in-library residuals -- the practically relevant case our analysis targets. 
Our goal is to make the symbolic and neural components complementary under the empirical inner product.

In their landmark paper, \citet{yin2021augmenting} present the APHYNITY framework for hybrid dynamical systems modeling when the symbolic structure (but not exact parameter values) is known a priori. 
Like a large body of related existing work~\citep{rackauckas2020universal,mouli2024metaphysica}, APHYNITY decomposes the (autonomous) vector field of an ordinary differential equation (ODE) as $f=f_{\mathrm{phy}}+f_{\mathrm{aug}}$, where $f_{\mathrm{phy}} \in \mathcal{F}_{\mathrm{phy}} = \operatorname{span}\{\phi_j\}_{j=1}^M$ captures dynamics within a predetermined library of ``symbolic'' functions $\{\phi_j\}_{j=1}^{M}$ (e.g., polynomials, trigonometric functions), while $f_{\mathrm{aug}}$ is supposed to capture the residual dynamics via flexible neural networks.
When the symbolic structure is fixed and $\mathcal{F}_{\mathrm{phy}}$ is a closed linear subspace, the two components are provably separated via simple $L^2$ regularization of $f_{\mathrm{aug}}$ (\cref{prop:fixed_symbolic}). 
In this setting, optimizing out $f_{\mathrm{aug}}$ reduces to an $L^2$ projection onto the fixed subspace $\mathcal{F}_{\mathrm{phy}}$, so the residual (and hence $f_{\mathrm{aug}}$) is orthogonal to $\mathcal{F}_{\mathrm{phy}}$, which underlies the analysis of \citet{yin2021augmenting}.

\begin{figure}
\centering
\includegraphics[width=\linewidth]{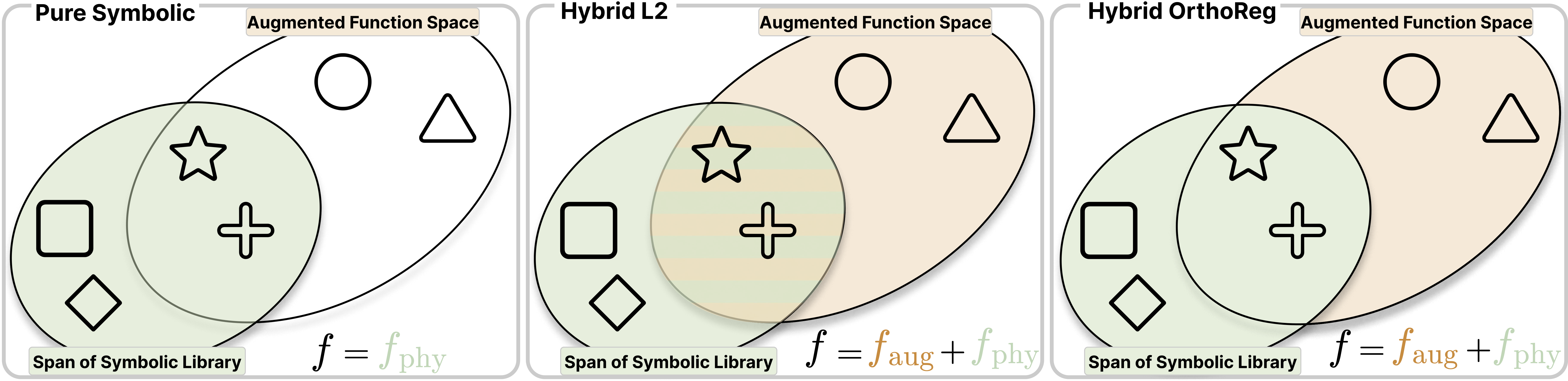}
\caption{\textbf{Symbolic--neural decompositions under library mismatch.}
The symbolic and augmented function spaces can overlap. 
Pure symbolic models are restricted to the symbolic span; standard hybrids may redundantly explain the shared region with both $f_{\mathrm{phy}}$ and $f_{\mathrm{aug}}$; OrthoReg discourages this redundancy by pushing $f_{\mathrm{aug}}$ away from $\mathcal{F}_{\mathrm{phy}}$, yielding a complementary decomposition.}
\label{fig:concept}
\end{figure}

APHYNITY gives a projection-based view of hybrid modeling when the physical model class is fixed: 
the augmentation has a minimal norm subject to fitting the observed dynamics, and in finite-dimensional \(L^2\) subspaces the residual is orthogonal to the physical component. 
This fixed-class picture changes when the symbolic component is learned through gradient-based sparse discovery. 
Continuous sparsity penalties, such as \(L^1\)-regularized library learning, shrink coefficients away from the unregularized projection and can leave residual structure inside the symbolic span. 
A neural augmentation can then relearn symbolic directions even when small in norm. 
OrthoReg targets this regime by directly penalizing empirical overlap between the augmentation and the symbolic library.

\begin{example}[Counterexample: sparsity forces overlap]
\label{ex:intro_sparsity_overlap}
Let $\mathcal{F}$ be a function space with inner product $\langle \cdot, \cdot \rangle_{\mathcal{D}}$ and induced norm $\|\cdot\|_{\mathcal{D}}$.
Let $\phi_1,\phi_2\in\mathcal{F}$ be orthonormal w.r.t.\ $\langle \cdot, \cdot \rangle_{\mathcal{D}}$ and $f=\phi_1+\phi_2$.
Consider
\begin{equation*}
\min_{w\in\mathbb{R}^2,\,g\in\mathcal{F}}\ \|f-(w_1\phi_1+w_2\phi_2+g)\|_{\mathcal{D}}^2+\lambda_2\|g\|_{\mathcal{D}}^2+\mu\|w\|_1.
\end{equation*}
If $0<\mu<\tfrac{2\lambda_2}{1+\lambda_2}$, the unique minimizer satisfies $w_1=w_2=1-\tfrac{\mu(1+\lambda_2)}{2\lambda_2}$ and $g=\tfrac{\mu}{2\lambda_2}(\phi_1+\phi_2)$, hence $g\not\perp \operatorname{span}\{\phi_1,\phi_2\}$ although $\operatorname{span}\{\phi_1,\phi_2\}$ contains $f$.
\end{example}

Thus, hybrid dynamical models that include symbolic discovery require explicit mechanisms to prevent neural-symbolic overlap.
\Cref{fig:concept} illustrates this failure mode and how OrthoReg restores a clean symbolic--neural split.

In this work, we introduce a theoretically grounded and practically effective method to learn hybrid dynamical systems, where the mechanistic component is discovered from data via symbolic regression while ensuring that the residual neural component remains orthogonal to the symbolic part.
The key contribution is to bridge the gap between theoretical projection theory arguments and practical optimization to leverage both symbolic and data-driven modeling.
Concretely, we provide:
\begin{itemize}[leftmargin=*,itemsep=0pt]
    \item \textbf{Theoretical analysis} of symbolic--neural overlap under sparse library-based discovery: $L^2$ regularization can induce overlap between the two components (\cref{thm:l2_failure}), whereas OrthoReg yields a clean in-library / orthogonal-complement split under the empirical inner product (\cref{thm:error_decomp}).
    \item An \textbf{algorithmic solution} realizing the orthogonality constraint within the joint sparse-discovery objective.
    \item \textbf{Empirical validation} on benchmark dynamical systems with partial library mismatch, showing improved out-of-distribution generalization and sparse symbolic recovery.
    \footnote{Code: \url{https://github.com/richtertill/OrthoReg}.}
\end{itemize}

\section{Related Work}
We review data-driven methods for learning governing dynamics and the interpretability--expressiveness trade-off motivating OrthoReg.

\paragraph{(Dynamic) symbolic regression.}
Symbolic regression recovers closed-form expressions from data using genetic programming \citep{koza1994genetic,schmidt2009distilling}, neural-guided search \citep{petersen2021deep,udrescu2020ai}, or sparse library-based regression such as SINDy \citep{brunton2016discovering}. 
Sparse-library methods are highly interpretable when the true dynamics lie in the candidate set, but their expressiveness is limited by library design. 
Recent work improves symbolic discovery through physical constraints such as unit consistency \citep{tenachi2023deep}, transformer-based symbolic generation \citep{lample2019deep,biggio2021neural,valipour2021symbolicgpt,kamienny2022endtoend,vastl2024symformer}, extensions to ODEs and trajectory data \citep{becker2023predicting,d2023odeformer,sun2023symbolic}, and methods for noisy, sparse, distributional, high-dimensional, or guided discovery settings \citep{qian2022d,dakhmouche2025robust,liang2025finite,tian2025interactive,hu2025learning}. 
However, when relevant effects are not representable by a compact symbolic library, purely symbolic models either fail or lose interpretability, motivating hybrid decompositions with a symbolic core and a flexible residual.

\paragraph{Physics-informed neural networks.}
PINNs \citep{raissi2019physics} embed differential equations as soft constraints and Universal ODEs \citep{rackauckas2020universal} parameterize unknown vector-field components with neural networks; 
surveys \citep{cuomo2022scientific,hao2022physics} position these as central scientific-ML paradigms. 
Both rely on a fixed symbolic part and do not address overlap when the symbolic component is itself discovered under sparsity.

\paragraph{Neural, symbolic, and hybrid methods.}
Hybrid methods combine symbolic interpretability with neural expressiveness.
Gray-box discovery has been studied for PDEs via sparse regression \citep{rudy2017data} and physics-informed symbolic objectives \citep{chen2021physics,KIYANI2023116258}.
For ODEs, APHYNITY formalizes $f=f_{\mathrm{phy}}+f_{\mathrm{aug}}$ but assumes fixed symbolic structure \citep{yin2021augmenting}; related work alternates neural learning with symbolic distillation \citep{grigorian2024hybrid} or studies PINN/PIKAN optimization for gray-box identification \citep{daryakenari2025representation}.
In related feedforward-control and nonlinear system-identification work, orthogonality between prescribed physics-based and learned additive components is enforced explicitly, either via a projection-based regularizer \citep{kon2022physics,gyorok2025orthogonal} or by an orthogonal-by-construction parametrization \citep{gyorok2026orthogonal}.
Pure neural ODE models can encode sparsity \citep{aliee2022sparsity}, conservation/manifold constraints \citep{greydanus2019hamiltonian,matsubara2023finde,white2023stabilized}, or meta-learned physics--ML trade-offs \citep{mouli2024metaphysica}, but generally lack explicit symbolic recovery.
In contrast, OrthoReg targets the overlap problem that arises when the symbolic component is itself \emph{discovered} jointly with a neural residual.

\section{Background and Problem Setting}
\label{sec:background}


Let $\mathcal{F}$ be a class of measurable vector fields $f:\mathbb{R}^n\to\mathbb{R}^n$ with finite $L^2(\nu)$ norm, where $\nu$ is the data-generating measure on states. 
We use the population inner product  
$\langle f,g\rangle_\nu = \mathbb{E}_{x\sim\nu}[f(x)^\top g(x)]$ 
and the empirical semi-inner product 
$\langle f,g\rangle_{\mathcal{D}} = \tfrac{1}{N}\sum_{i=1}^N f(x_i)^\top g(x_i)$ 
on the sampled states $\mathcal{D}=\{x_i\}_{i=1}^N$, with induced norm and seminorm $\|\cdot\|_\nu$ and $\|\cdot\|_{\mathcal{D}}$. 
For a closed linear subspace $\mathcal{F}_{\mathrm{phy}}\subseteq\mathcal{F}$, we denote the corresponding population projection by $P_{\mathcal{F}_{\mathrm{phy}}}^{\nu}$; empirically, $P_{\mathcal{F}_{\mathrm{phy}}}^{\mathcal{D}}$ denotes the projection under $\|\cdot\|_{\mathcal{D}}$, understood up to equivalence on $\mathcal{D}$.

The functions $f \in \mathcal{F}$ are interpreted as vector fields of autonomous, first-order differential equations
\begin{align*}
    \frac{\mathrm{d}x}{\mathrm{d}t} = f(x), \qquad \text{with solution trajectories } x: \mathbb{R} \to \mathbb{R}^n\:.
\end{align*}

Following prior work \citep{yin2021augmenting,rackauckas2020universal}, we assume an additive decomposition
\begin{align*}
    f = f_{\mathrm{phy}} + f_{\mathrm{aug}}, \qquad 
    f_{\mathrm{phy}} \in \mathcal{F}_{\mathrm{phy}},\; f_{\mathrm{aug}} \in \mathcal{F}
\end{align*}
of vector fields of interest into a ``physical'' (or symbolic/mechanistic) component and an ``augmented'' (or neural/residual) component.
The space $\mathcal{F}_{\mathrm{phy}} \subseteq \mathcal{F}$ of candidate symbolic components is typically restricted to functions that can be represented in closed form using known functions, so as to be amenable to direct interpretation and dissemination by humans.

Most existing methods assume $f_{\mathrm{phy}}$ to be either known exactly, or to be given as a parametric family in which only a (usually small) set of parameters is unknown.
Practically, this is often implemented via a linear combination of non-linear basis functions:
\begin{equation}
\mathcal{F}_{\mathrm{phy}}
= \left\{ \sum_{j=1}^M \alpha_j \phi_j \;\middle|\; \alpha_j \in \mathbb{R} \text{ for } j \in \{1,\ldots, M\} \right\},
\label{eq:lindict}
\end{equation}
where the dictionary functions $\phi_j : \mathbb{R}^n \to \mathbb{R}^n$ are fixed. The dynamics governing most real systems are not perfectly described by such simple closed-form expressions, but contain higher-order effects or complex interactions that are rarely captured by simple interpretable mathematical expressions.
To capture such residual effects, the augmentation $f_{\mathrm{aug}} \in \mathcal{F}$ is supposed to be flexible and expressive, albeit potentially not easily interpretable.
Hence, a natural choice to represent $f_{\mathrm{aug}}$ is via flexible function approximators such as neural networks, giving rise to the term ``neural component.''
Crucially, the neural component should \emph{only capture effects that cannot be captured by the symbolic component}.

In the current formulation, one could simply set $f_{\mathrm{aug}} \equiv f$ and $f_{\mathrm{phy}} \equiv 0$.
However, this would undermine the entire idea of hybrid modeling.
When the physical model class is fixed, \citet{yin2021augmenting} provide thorough theoretical guarantees showing that a relatively simple norm-based regularization scheme is sufficient to ensure that $f_{\mathrm{aug}}$ ``only captures what is necessary, but not more.''
A simplified vector-field version of this norm-regularized principle is
\begin{equation}
    \min_{f_{\mathrm{phy}} \in \mathcal{F}_{\mathrm{phy}}, f_{\mathrm{aug}} \in \mathcal{F}}
    \; \| f - f_{\mathrm{phy}} - f_{\mathrm{aug}} \|_{\mathcal{D}}^2
    + \lambda_2 \| f_{\mathrm{aug}} \|_{\mathcal{D}}^2\:.
    \label{eq:aphynity}
\end{equation}
For fixed $f_{\mathrm{phy}}$, writing $r=f-f_{\mathrm{phy}}$, the minimizer over $f_{\mathrm{aug}}$ is
\begin{align*}
    \hat f_{\mathrm{aug}} = \frac{1}{1+\lambda_2}r
    = \frac{1}{1+\lambda_2}(f - f_{\mathrm{phy}}).
\end{align*}
Substituting back gives
\begin{align*}
    \min_{f_{\mathrm{aug}}\in\mathcal F}
    \left[
    \| f - f_{\mathrm{phy}} - f_{\mathrm{aug}} \|_{\mathcal{D}}^2
    + \lambda_2 \| f_{\mathrm{aug}} \|_{\mathcal{D}}^2
    \right]
    =
    \frac{\lambda_2}{1+\lambda_2}
    \| f - f_{\mathrm{phy}} \|_{\mathcal{D}}^2.
\end{align*}
Thus, for $\lambda_2>0$, optimizing over $f_{\mathrm{phy}}$ is equivalent to empirical least-squares projection onto $\mathcal{F}_{\mathrm{phy}}$.
If $\mathcal{F}_{\mathrm{phy}}$ is a closed linear subspace, the Hilbert-space projection theorem gives $P_{\mathcal{F}_{\mathrm{phy}}}^{\mathcal{D}}(f)$ up to $\mathcal D$-equivalence, and the residual $f-P_{\mathcal{F}_{\mathrm{phy}}}^{\mathcal{D}}(f)$, hence $\hat f_{\mathrm{aug}}$, is orthogonal to $\mathcal{F}_{\mathrm{phy}}$ under $\langle\cdot,\cdot\rangle_{\mathcal D}$ \citep{lax2014functional}.
APHYNITY extends this projection view beyond linear subspaces via existence and uniqueness guarantees under \emph{proximinality} and \emph{Chebyshevness} \citep{yin2021augmenting}; see \cref{app:l2_failure}.

A natural extension beyond a fully known $f_{\mathrm{phy}}$, or the structure being known up to a small set of parameters, is to discover the symbolic component itself from a larger dictionary, as in SINDy \citep{brunton2016discovering}.
After fixing the candidate basis functions $\{\phi_j\}_{j=1}^M$ we select only a small support set $S \subset \{1,\dots,M\}$ of basis functions that enter the expression with non-zero coefficients. The induced function space is
\begin{align*}
    \mathcal{F}_{\mathrm{phy}}(S) := \operatorname{span}\{\phi_j \mid j \in S\}.
\end{align*}
In practice, the set $S$ is fitted via sparse regression methods (e.g., L1 regularization $\|\cdot \|_1$ or more involved iterated sparse regressions as in SINDy) to encourage small supports $S$.

While at first this appears to be a natural extension to APHYNITY, on closer inspection this changes the projection geometry underlying its argument.
When sparse discovery is implemented through continuous shrinkage penalties, the symbolic fit is no longer an unregularised projection onto the selected span; 
the residual can therefore retain components in the symbolic library, which the neural augmentation may relearn (cf.~\cref{fig:concept,ex:intro_sparsity_overlap}).


\textbf{This is the fundamental gap our work addresses:} expressive data-driven sparse symbolic discovery requires additional mechanisms to keep symbolic and neural components from overlapping. 
The closest faithful adaptation of APHYNITY to this setting is therefore a sparse symbolic component combined with an $L^2$-regularized neural augmentation, which we call the ``Hybrid $L^2$'' baseline.
\section{Method: OrthoReg for Hybrid Modeling}\label{sec:methods}

We observe state trajectories of an unknown dynamical system and aim to learn an autonomous vector field $f:\mathbb{R}^n\to\mathbb{R}^n$ such that $\dot x=f(x)$.
Following the hybrid modeling setup in \cref{sec:background}, we model
\begin{equation}
    \hat f(x)=\hat f_{\mathrm{phy}}(x;w)+\hat f_{\mathrm{aug}}(x;\vartheta),
    \;
    \hat f_{\mathrm{phy}}(x;w)=\sum_{j=1}^M w_j\,\phi_j(x),
\end{equation}
where $\{\phi_j\}_{j=1}^M$ is a fixed symbolic library, $w$ is sparse, and $\hat f_{\mathrm{aug}}(\cdot;\vartheta)$ is a flexible neural augmentation.

\paragraph{Learning from trajectories.}
The data consists of state observations over time (trajectories). Depending on the setting, derivative targets may or may not be available.
We therefore consider two standard data-fit losses for learning $\hat f$. If (approximate) derivatives $y_i \approx \dot x(t_i)$ are available (e.g., from a simulator or numerical differentiation), we can fit $\hat f$ by vector-field regression
\begin{equation}
    \mathcal{L}_{\mathrm{vf}}(w,\vartheta)
    =
    \frac{1}{N}\sum_{i=1}^N \big\|y_i-\hat f(x_i)\big\|^2.
    \label{eq:fit_loss_vf}
\end{equation}

If only states are observed, we fit by one-step prediction.
Let $\mathcal{T}$ be the set of observed one-step transitions $(x_t,x_{t+1},\Delta t_t)$ (potentially with irregular $\Delta t_t$) and let $\Psi_{\Delta t}(\cdot;\hat f)$ denote a one-step ODE solver (e.g., explicit Euler $\Psi_{\Delta t}(x;\hat f)=x+\Delta t\,\hat f(x)$ or a Runge--Kutta method).
The one-step loss is
\begin{equation}
    \mathcal{L}_{\mathrm{step}}(w,\vartheta)
    =
    \frac{1}{|\mathcal{T}|}\sum_{(x_t,x_{t+1},\Delta t_t)\in\mathcal{T}}
    \big\|x_{t+1}-\Psi_{\Delta t_t}(x_t;\hat f)\big\|^2.
    \label{eq:fit_loss_step}
\end{equation}

In the following, $\mathcal{L}_{\mathrm{fit}}$ denotes either $\mathcal{L}_{\mathrm{vf}}$ or $\mathcal{L}_{\mathrm{step}}$.
We train with $\mathcal{L}_{\mathrm{vf}}$ (state-only training is also compatible with OrthoReg; see \cref{app:limitations}), as derivative-space training was empirically more stable, and evaluate in both derivative and state space.
OrthoReg uses the empirical inner product $\langle \cdot, \cdot \rangle_{\mathcal{D}}$ over observed states $\mathcal{D}=\{x_i\}_{i=1}^N$, defined in \cref{sec:background}.

\subsection{OrthoReg Objective}

OrthoReg encourages a non-redundant decomposition by penalizing correlation between the neural augmentation and the symbolic library.
Concretely, we penalize the (squared) empirical inner products

\begin{equation}
    \mathcal{L}_{\mathrm{reg}}^{\perp}(\vartheta)
    =
    \lambda \sum_{j=1}^M\left\langle \hat{f}_{\mathrm{aug}}(\cdot;\vartheta), \phi_j\right\rangle_{\mathcal{D}}^2,
    \label{eq:orthoreg_penalty}
\end{equation}

so that $\mathcal{L}_{\mathrm{reg}}^{\perp}=0$ implies $\hat f_{\mathrm{aug}}\perp\operatorname{span}\{\phi_1,\dots,\phi_M\}$ w.r.t.\ $\langle\cdot,\cdot\rangle_{\mathcal{D}}$.
Let $\mu\ge 0$ control symbolic sparsity and $\lambda\ge 0$ the orthogonality strength, then
the full OrthoReg training objective is

\begin{equation}
    \min_{w,\vartheta}\;
    \mathcal{L}_{\mathrm{fit}}(w,\vartheta)
    + \mu \|w\|_1
    + \mathcal{L}_{\mathrm{reg}}^{\perp}(\vartheta),
    \label{eq:orthoreg_objective}
\end{equation}


\begin{figure}[h]
\centering
\fbox{\begin{minipage}{0.7\linewidth}
\textbf{Algorithm 1: OrthoReg training (vector-field regression).}\\[2pt]
\textbf{Input:} samples $\{(x_i,y_i)\}_{i=1}^N$, library $\{\phi_j\}_{j=1}^M$, weights $\lambda,\mu\ge 0$.\\
\textbf{Initialize} $w\in\mathbb{R}^M$, neural parameters $\vartheta$.\\
For each epoch and minibatch $\mathcal{B}\subset\{1,\ldots,N\}$ of size $B$:
\begin{enumerate}[leftmargin=*,itemsep=0pt,topsep=2pt]
  \item $\hat f(x_i) = \sum_{j=1}^M w_j\,\phi_j(x_i) + \hat f_{\mathrm{aug}}(x_i;\vartheta)$.
  \item $L_{\mathrm{fit}} = \tfrac{1}{B}\sum_{i\in\mathcal{B}}\|y_i-\hat f(x_i)\|^2$.
  \item $g_j = \tfrac{1}{B}\sum_{i\in\mathcal{B}}\hat f_{\mathrm{aug}}(x_i;\vartheta)^\top \phi_j(x_i)$, \quad
  $L_{\mathrm{orth}} = \lambda\sum_{j=1}^M g_j^2$.
  \item $L_{\mathrm{sparse}} = \mu\|w\|_1$.
  \item Update $(\vartheta,w)$ by autodiff on $L_{\mathrm{fit}}+L_{\mathrm{orth}}+L_{\mathrm{sparse}}$.
\end{enumerate}
\end{minipage}}
\end{figure}

The guarantee in \cref{thm:penalty_bound_main} is tied to the additive objective in \cref{eq:orthoreg_objective}: 
the orthogonality penalty must act on the additive residual $\hat f_{\mathrm{aug}}$ at the same empirical inputs as the data-fit. 
Staged or alternating updates of the same objective fall within the same formal scope when they reach a solution satisfying the theorem assumptions, whereas unconstrained residual fitting after shrinkage-based symbolic learning and compositional architectures of the form $T_w+N_\vartheta\circ T_w$ do not (\cref{prop:seq_two_stage_failure,app:sequential_extension}).


\subsection{Implementation and Computational Considerations}
\label{subsec:meth_impl}

\paragraph{Implementation.}
Algorithm~1 summarises training. 
We optimise the continuous $L^1$-regularised objective with Adam and report sparsity after thresholding small coefficients. 
The orthogonality penalty costs $\mathcal{O}(B\,n\,M)$ per batch; 
in mini-batch training, finite-batch squaring of the inner-product estimate introduces a small bias, so our experiments use full- or large-batch estimates unless stated otherwise. 
Hyperparameter selection, feature scaling, and the autonomous-augmentation caveat for non-autonomous data are described in \cref{app:hyperparam_selection,app:non_autonomy_caveat}.

\subsection{Theoretical Guarantees}

We first show that $L^2$ regularization does not prevent -- and can in fact induce -- overlap between symbolic and neural components under sparse selection (\cref{thm:l2_failure}, \cref{app:l2_failure}). 
The OrthoReg objective restores control via two structural results: 
a penalty bound showing that empirical overlap is controlled at the optimum and vanishes as $\lambda\to\infty$, and a triangular error decomposition (the in-library / orthogonal-complement split) that, under the empirical inner product, serves as a formal lens for the symbolic--neural decomposition.

\begin{theorem}[Penalty bound for empirical orthogonality]
\label{thm:penalty_bound_main}
Assume $\lambda>0$ and that \cref{eq:orthoreg_objective} admits a global minimizer $(w^*,\vartheta^*)$. Let $w_{\mathrm{sym}}\in\arg\min_w\,\mathcal{L}_{\mathrm{fit}}(w,\vartheta_0)+\mu\|w\|_1$ denote the best pure-symbolic fit at $\hat f_{\mathrm{aug}}\equiv 0$ (parameters $\vartheta_0$). Then
\begin{equation*}
\sum_{j=1}^M\langle \hat{f}_{\mathrm{aug}}(\cdot;\vartheta^*), \phi_j\rangle_{\mathcal{D}}^2 \le \tfrac{1}{\lambda}\bigl(\mathcal{L}_{\mathrm{fit}}(w_{\mathrm{sym}},\vartheta_0)+\mu\|w_{\mathrm{sym}}\|_1\bigr),
\end{equation*}
so empirical orthogonality is enforced as $\lambda\to\infty$.
\end{theorem}
\begin{proof}[Proof idea]
The orthogonality term is a quadratic penalty: comparing the optimum to the baseline $\hat f_{\mathrm{aug}}\equiv 0$ yields a $1/\lambda$ bound on constraint violation. A full proof is given in \cref{thm:ortho_stationary}.
\end{proof}

\paragraph{Error decomposition.}
Beyond the penalty bound, orthogonality has a concrete geometric consequence for approximation.
In the idealized vector-field regression setting, and when the learned augmentation is empirically orthogonal to $\mathcal F_{\mathrm{phy}}$, the squared error admits an orthogonal decomposition:
\begin{equation}
\|f-(\hat f_{\mathrm{phy}}+\hat f_{\mathrm{aug}})\|_{\mathcal{D}}^2
= \|P_{\mathcal{F}_{\mathrm{phy}}}^{\mathcal{D}}(f)-\hat f_{\mathrm{phy}}\|_{\mathcal{D}}^2
+ \|f-P_{\mathcal{F}_{\mathrm{phy}}}^{\mathcal{D}}(f)-\hat f_{\mathrm{aug}}\|_{\mathcal{D}}^2,
\end{equation}
Thus, the symbolic error is the gap to the in-library projection $P_{\mathcal{F}_{\mathrm{phy}}}^{\mathcal{D}}(f)$, while the neural error is the gap to the empirical orthogonal-complement component $f-P_{\mathcal{F}_{\mathrm{phy}}}^{\mathcal{D}}(f)$; see \cref{thm:error_decomp}.
Under the empirical inner product, this decomposition serves as a formal lens for the symbolic--neural split: 
each component's error is the gap to its half of the orthogonal direct sum.

\paragraph{Comparison to $L^2$ regularization.}
$L^2$ regularization controls only the \emph{magnitude} of the augmentation through the decomposition
\begin{equation*}
    \|\hat{f}_{\mathrm{aug}}\|_{\mathcal{D}}^2
    =
    \|P_{\mathcal{F}_{\mathrm{phy}}}^{\mathcal{D}}(\hat{f}_{\mathrm{aug}})\|_{\mathcal{D}}^2
    +
    \|\hat{f}_{\mathrm{aug}} - P_{\mathcal{F}_{\mathrm{phy}}}^{\mathcal{D}}(\hat{f}_{\mathrm{aug}})\|_{\mathcal{D}}^2,
\end{equation*}
where the equality follows from orthogonal decomposition in inner product spaces \citep{rudin1987real}. Even when the total is small, the component $P_{\mathcal{F}_{\mathrm{phy}}}^{\mathcal{D}}(\hat{f}_{\mathrm{aug}})$ can be non-zero, allowing neural--symbolic overlap. When $\mathcal{F}_{\mathrm{phy}}$ is learned through sparsity constraints, this overlap can occur even in realizable settings, motivating explicit orthogonality regularization (\cref{app:l2_failure}).

\paragraph{Empirical-to-population transfer.}
OrthoReg controls overlap under the empirical inner product used for training. 
A formal population guarantee for the learned augmentation would require uniform convergence over the augmentation hypothesis class; we leave this to future work (\cref{app:limitations}).

\section{Experiments}
\label{sec:experiments}

\begin{table}[h]
\centering
\small
\setlength{\tabcolsep}{5pt}
\caption{
Damped Pendulum results. 
Scale factors $\mathrm{e}{k}$ inside parentheses are powers of ten; $x$/$\dot x$ are state/derivative spaces.}
\label{tab:sweet_spot_results}
\begin{tabular}{lccc}
\toprule
\textbf{Metric} & \textbf{Pure} & \textbf{$L^2$} & \textbf{OrthoReg} \\
\midrule
\multicolumn{4}{c}{\textbf{In-Distribution Performance}} \\
\midrule
MSE(${\scriptstyle \dot{x}_\mathrm{ID},\ \mathrm{e}{-2}}$) {\scalebox{0.8}{($\downarrow$)}} 
& $\bm{1.27\ {\scalebox{0.8}{$\pm 0.07$}}}$ 
& $1.41\ {\scalebox{0.8}{$\pm 0.05$}}$ 
& $1.89\ {\scalebox{0.8}{$\pm 0.01$}}$ \\
MSE(${\scriptstyle x_\mathrm{ID},\ \mathrm{e}{-2}}$) {\scalebox{0.8}{($\downarrow$)}} 
& $0.74\ {\scalebox{0.8}{$\pm 0.43$}}$ 
& $\bm{0.31\ {\scalebox{0.8}{$\pm 0.02$}}}$ 
& $3.65\ {\scalebox{0.8}{$\pm 0.36$}}$ \\
MSE(${\scriptstyle \dot{x}_\mathrm{ID,ext},\ \mathrm{e}{2}}$) {\scalebox{0.8}{($\downarrow$)}} 
  & $19.45\ {\scalebox{0.8}{$\pm 12.89$}}$ 
  & $5.00\ {\scalebox{0.8}{$\pm 2.42$}}$ 
  & $\bm{0.13\ {\scalebox{0.8}{$\pm 0.07$}}}$ \\
\midrule
\multicolumn{4}{c}{\textbf{Out-of-Distribution Performance}} \\
\midrule
MSE(${\scriptstyle x_\mathrm{OOD,T2},\ \mathrm{e}{0}}$) {\scalebox{0.8}{($\downarrow$)}} 
  & $1.02\ {\scalebox{0.8}{$\pm 0.07$}}$ 
  & $1.07\ {\scalebox{0.8}{$\pm 0.02$}}$ 
  & $1.01\ {\scalebox{0.8}{$\pm 0.01$}}$ \\
MSE(${\scriptstyle \dot{x}_\mathrm{OOD,T3},\ \mathrm{e}{2}}$) {\scalebox{0.8}{($\downarrow$)}} 
& $34.00\ {\scalebox{0.8}{$\pm 22.22$}}$
& $7.90\ {\scalebox{0.8}{$\pm 4.02$}}$
& $\bm{0.02\ {\scalebox{0.8}{$\pm 0.00$}}}$ \\
\midrule
\multicolumn{4}{c}{\textbf{System Identification Quality}} \\
\midrule
F1 {\scalebox{0.8}{($\uparrow$)}} 
& $0.43\ {\scalebox{0.8}{$\pm 0.09$}}$
& $0.61\ {\scalebox{0.8}{$\pm 0.03$}}$
& $\bm{0.93\ {\scalebox{0.8}{$\pm 0.15$}}}$ \\
\#Terms {\scalebox{0.8}{($\downarrow$)}} 
& $11.4\ {\scalebox{0.8}{$\pm 2.9$}}$
& $6.8\ {\scalebox{0.8}{$\pm 0.4$}}$
& $\bm{3.6\ {\scalebox{0.8}{$\pm 1.3$}}}$ \\
\bottomrule
\end{tabular}
\end{table}

We evaluate OrthoReg across four dynamical systems of different complexities: a modified damped pendulum, a Lotka--Volterra predator-prey system, a time-modulated SIR epidemiological model, and a Duffing oscillator. 
Our evaluation focuses on two complementary metrics: 
(i) trajectory accuracy measured by normalized mean-squared error (MSE) on derivatives and integrated states\footnote{MSE values are normalized by the squared norm of the target signal for scale invariance.} on in-distribution, out-of-distribution, and extrapolation states; and 
(ii) symbolic recovery quality measured by F1 score and the number of symbolic terms. 
We compare three approaches: 
pure symbolic regression (SINDy), an $L^2$-regularized hybrid, and an OrthoReg-regularized hybrid. 
We report along four axes consistent with multi-axis evaluation principles in scientific ML \citep{ctf2025seismic}: initial-condition extrapolation, parametric extrapolation, regime/basin extrapolation, and robustness under irregular sampling and observation noise.

\subsection{Damped Pendulum: Missing Dynamics}

The modified damped pendulum extends the classical damped pendulum \citep{kharkongor2018resonance} with three forcing terms that lie outside the symbolic dictionary:
\smallskip
\begin{equation}
\label{eq:pend}
    \ddot{\theta} + \alpha\dot{\theta} + \omega_0^2\sin\theta
    \;=\;
    \beta_1\cos(3\theta) + \beta_2 e^{-\theta^2} + \beta_3 \tanh(\dot\theta),
\end{equation}
with $(\beta_1,\beta_2,\beta_3)=(0.3,0.25,0.15)$.
The dictionary used by the symbolic component is $\mathrm{PolynomialLibrary}(\mathrm{degree}=2)\oplus\mathrm{FourierLibrary}(n_{\mathrm{freq}}=1)$, which contains $\sin\theta,\cos\theta,\sin(2\theta),\cos(2\theta)$ and monomials of $(\theta,\dot\theta)$ up to total degree two; the three forcing terms above lie outside this dictionary by construction.


\Cref{tab:sweet_spot_results} reports performance under medium missing dynamics ($\beta_1=0.3,\beta_2=0.25,\beta_3=0.15$ in \cref{eq:pend}). 
Derivative and state MSE quantify trajectory fit, and F1 measures symbolic recovery against ground-truth terms. 
OrthoReg yields the strongest symbolic recovery (highest F1, fewest terms) and the largest reductions in OOD-derivative error, at the cost of higher in-distribution state MSE. 
Quantitatively, OOD-derivative error drops by more than two orders of magnitude relative to the $L^2$ baseline while in-distribution state MSE is higher than the $L^2$ baseline. 
This trade-off is consistent with the symbolic / orthogonal-residual split established by \cref{thm:error_decomp}: 
enforcing the split sacrifices some in-distribution fit in exchange for components that transfer to OOD evaluations.


\subsection{Comparison with Pure Neural Baselines}

\begin{table}[h]
\centering
\small
\setlength{\tabcolsep}{5pt}
\caption{Baseline comparison on the modified damped pendulum. OrthoReg uniquely provides symbolic recovery while remaining competitive on prediction.}

\label{tab:baseline}
\begin{tabular}{lccc}
\toprule
\textbf{Metric} & \textbf{PINN} & \textbf{Universal ODE} & \textbf{OrthoReg} \\
\midrule
MSE(${\scriptstyle \dot{x}_\mathrm{ID},\ \mathrm{e}{-2}}$) {\scalebox{0.8}{($\downarrow$)}} 
  & $\bm{1.58\ {\scalebox{0.8}{$\pm 0.00$}}}$ 
  & $1.61\ {\scalebox{0.8}{$\pm 0.06$}}$ 
  & $1.89\ {\scalebox{0.8}{$\pm 0.01$}}$ \\
  MSE(${\scriptstyle x_\mathrm{OOD,T2},\ \mathrm{e}{0}}$) {\scalebox{0.8}{($\downarrow$)}} 
  & $1.02\ {\scalebox{0.8}{$\pm 0.01$}}$ 
  & $1.07\ {\scalebox{0.8}{$\pm 0.07$}}$ 
  & $1.01\ {\scalebox{0.8}{$\pm 0.01$}}$ \\
  
  MSE(${\scriptstyle \dot{x}_\mathrm{OOD,T3},\ \mathrm{e}{2}}$) {\scalebox{0.8}{($\downarrow$)}} 
  & $0.94\ {\scalebox{0.8}{$\pm 0.00$}}$ 
  & $1.02\ {\scalebox{0.8}{$\pm 0.18$}}$ 
  & $\bm{0.02\ {\scalebox{0.8}{$\pm 0.00$}}}$ \\
  
\midrule
F1 {\scalebox{0.8}{($\uparrow$)}} 
  & -- 
  & -- 
  & $\bm{0.93\ {\scalebox{0.8}{$\pm 0.15$}}}$ \\
\bottomrule
\end{tabular}
\end{table}

We also compare to pure neural approaches in \cref{tab:baseline}, including PINNs \citep{raissi2019physics} and Universal Differential Equations \citep{rackauckas2020universal}. 
While these baselines achieve competitive trajectory fitting, they cannot recover symbolic components. OrthoReg recovers a sparse symbolic component and is competitive on OOD$_{T2}$ and OOD$_{T3}$ derivative MSE. 
Implementation details are in \cref{app:baseline_implementation}.
Canonical PySINDy STLSQ (same library and splits as the main experiments) is summarised in \cref{app:stlsq_baseline}: it improves symbolic support recovery relative to the Adam $L^1$ pure-symbolic rows when thresholding is the sparsity mechanism, but unstable rollouts in partial mismatch show why residual capacity remains necessary.

\begin{figure}[h]
\centering
\includegraphics[width=\linewidth, trim=7 7 7 7, clip]{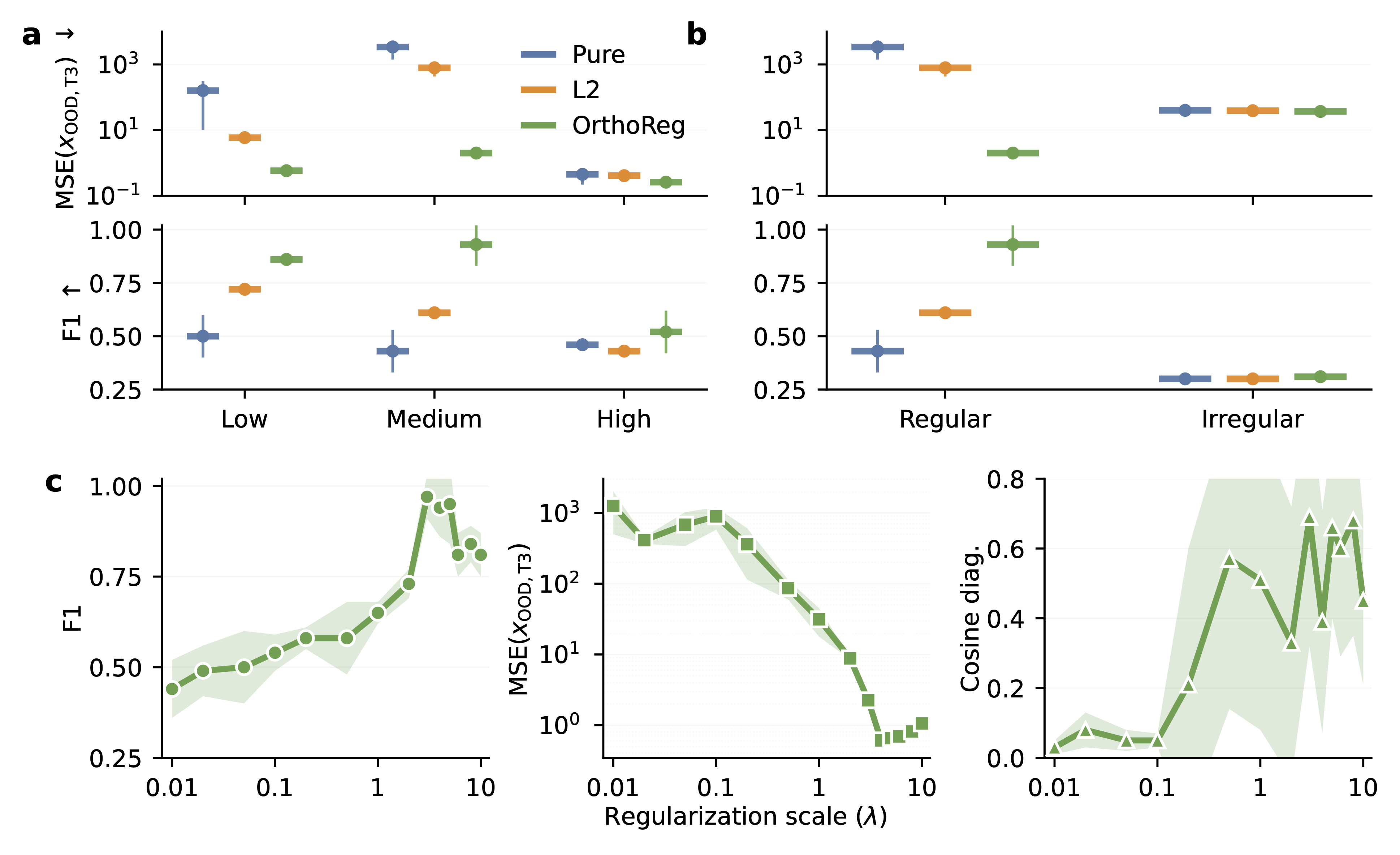}
\caption{\textbf{Ablations.} 
(a) OrthoReg is most effective under partial library mismatch. 
(b) Irregular sampling degrades all methods, while OrthoReg retains competitive OOD behavior. 
(c) Intermediate $\lambda$ gives the best trade-off between symbolic recovery and OOD error; right panel shows a residual--symbolic cosine diagnostic for OrthoReg to visualize separation.}

\label{fig:ablations}
\end{figure}

\subsection{Ablations}
\label{subsec:ablations}

We ablate three factors that probe when and why orthogonal regularization is beneficial: (i) the degree to which the true dynamics exceed the symbolic library, (ii) robustness to non-ideal (irregular) observation times, and (iii) the strength of the orthogonality penalty. \Cref{fig:ablations} summarizes the main trends. Complete quantitative results and additional ablation analyses are reported in \cref{app:additional_experiments}.

\vspace{1em}
\paragraph{Dataset difficulty (\cref{fig:ablations}a).}
OrthoReg's gains on stress-test generalization ($x_\mathrm{OOD,T3}$) are largest in the intermediate regime where the system partially exceeds the library; in low-missing regimes ($\beta_1=0.1, \beta_2=0.08, \beta_3=0.05$) differences shrink on simpler evaluations, and in high-missing regimes ($\beta_1=\beta_2=\beta_3=2.0$) all methods degrade with OrthoReg retaining a modest advantage. Full results in \cref{app:dataset_diff}.

\paragraph{Sampling scheme (\cref{fig:ablations}b).}
Under irregular sampling ($\Delta t \sim \text{Uniform}[0.01, 0.05]$ vs.\ regular $\Delta t = 0.02$), absolute performance drops across all methods, but OrthoReg retains a marginal F1 advantage and a competitive $x_\mathrm{OOD,T3}$. 
Full breakdown in \cref{app:sampling_ablation}.

\paragraph{Regularization strength (\cref{fig:ablations}c).} OrthoReg has an intermediate operating range: 
weak $\lambda$ under-enforces the orthogonality penalty, while overly strong $\lambda$ constrains the augmentation. 
We select $\lambda$ and $\mu$ by validation prediction loss on held-out in-distribution trajectories; 
F1 is tracked as a diagnostic only (\cref{app:hyperparam_selection}).
Right: cosine diagnostic visualizing residual--symbolic alignment within each method; regularization strength is assessed by F1 and OOD error.

\subsection{Duffing oscillator: regime-dependent generalization}

\begin{table}[h]
\centering
\small
\setlength{\tabcolsep}{5pt}
\caption{Duffing oscillator. Training on the positive basin; OOD T2 tests cross-basin (Regime-Ext), OOD T3 tests parameter shifts (Param-Ext).}
\label{tab:duffing_results_streamlined}
\begin{tabular}{lccc}
\toprule
\textbf{Metric} & \textbf{Pure} & \textbf{$L^2$} & \textbf{OrthoReg} \\
\midrule
\multicolumn{4}{c}{\textbf{Predictive Performance}} \\
\midrule
MSE(${\scriptstyle x_\mathrm{ID},\ \mathrm{e}{-1}}$) {\scalebox{0.8}{($\downarrow$)}} 
  & $\bm{4.75\ {\scalebox{0.8}{$\pm 0.13$}}}$ 
  & $5.36\ {\scalebox{0.8}{$\pm 0.14$}}$ 
  & $5.59\ {\scalebox{0.8}{$\pm 0.09$}}$ \\
MSE(${\scriptstyle x_\mathrm{OOD,T2},\ \mathrm{e}{0}}$) {\scalebox{0.8}{($\downarrow$)}} 
  & $11.03\ {\scalebox{0.8}{$\pm 6.09$}}$ 
  & $7.46\ {\scalebox{0.8}{$\pm 6.20$}}$ 
  & $\bm{4.93\ {\scalebox{0.8}{$\pm 0.35$}}}$ \\
MSE(${\scriptstyle x_\mathrm{OOD,T3},\ \mathrm{e}{-1}}$) {\scalebox{0.8}{($\downarrow$)}} 
  & $3.70\ {\scalebox{0.8}{$\pm 0.21$}}$ 
  & $\bm{3.44\ {\scalebox{0.8}{$\pm 0.12$}}}$ 
  & $3.46\ {\scalebox{0.8}{$\pm 0.05$}}$ \\
\midrule
\multicolumn{4}{c}{\textbf{System Identification Quality}} \\
\midrule
F1 {\scalebox{0.8}{($\uparrow$)}} 
  & $0.40\ {\scalebox{0.8}{$\pm 0.05$}}$ 
  & $0.59\ {\scalebox{0.8}{$\pm 0.04$}}$ 
  & $\bm{0.63\ {\scalebox{0.8}{$\pm 0.05$}}}$ \\
\#Terms {\scalebox{0.8}{($\downarrow$)}} 
  & $11.0\ {\scalebox{0.8}{$\pm 1.4$}}$ 
  & $3.8\ {\scalebox{0.8}{$\pm 0.4$}}$ 
  & $\bm{3.4\ {\scalebox{0.8}{$\pm 0.5$}}}$ \\
\bottomrule
\end{tabular}
\end{table}

We study the unforced Duffing oscillator, a multistable system whose cross-basin behavior probes whether models capture global structure beyond the observed basin~\citep{goring2024out}, and train on trajectories from the positive basin only. 
Here orthogonality is helpful because the symbolic library carries the global polynomial structure of the dynamics; 
discouraging the neural residual from absorbing those terms lets the symbolic component transfer across basins. 
\Cref{tab:duffing_results_streamlined} and \cref{fig:duffing} show that in-distribution all methods are comparable, while on cross-basin OOD T2 OrthoReg achieves the lowest error and strongest symbolic recovery; 
under parameter shifts (OOD T3) differences are within noise (details in \cref{app:subsec_duffing}). 
The gain is largest where evaluation stresses global state-space structure rather than raw fit, consistent with the orthogonal split established by \cref{thm:error_decomp}.

\begin{figure}[t]
\centering
\includegraphics[width=0.99\linewidth]{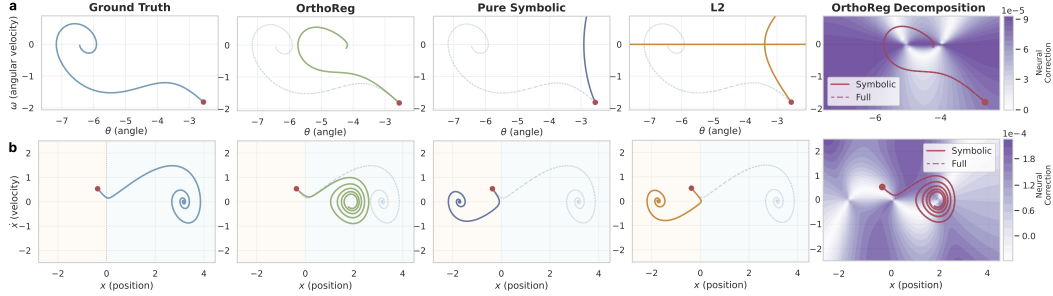}
\caption{
(a) Damped pendulum in $(\theta,\omega)$, with $\omega:=\dot\theta$. (b) Duffing in $(x,\dot{x})$ with shaded basins. Solid: inferred; dashed: ground truth. OrthoReg captures global trends; pure symbolic and $L^2$ distort the dynamics. Rightmost panels show the OrthoReg decomposition, highlighting the neural residual.}
\label{fig:duffing}
\end{figure}

\subsection{Cross-System Validation}

As auxiliary stress tests beyond the pendulum and Duffing, we evaluate OrthoReg on a Lotka--Volterra predator-prey system (coupled polynomial dynamics) and a time-modulated SIR model (severe library mismatch and state-dependent timescales); 
full results are in \cref{app:cross_system_validation}. 
We read LV as a low-headroom sanity check (\cref{tab:lv_results_streamlined}: methods are comparable, with a small F1 gain for OrthoReg) and SIR as a misspecification stress test (\cref{tab:sir_results_streamlined}: OrthoReg trades ID and OOD-T2 accuracy for the strongest OOD-T3 extrapolation and the sparsest symbolic recovery -- highest F1 and fewest terms). 
The pattern is consistent with the symbolic--neural split.

\subsection{Summary}

Across systems, OrthoReg is most effective in the regime it is designed for: partial library misspecification, where the symbolic library captures part of the dynamics but leaves structured residual effects. 
In this setting, OrthoReg improves symbolic recovery and reduces out-of-distribution derivative error, while accepting some loss in in-distribution fit. 
The gains persist under irregular sampling, but diminish under heavy observation noise and when the missing dynamics are strongly entangled with library features (\cref{app:noise_robustness}). 
These results support the view of OrthoReg as a decomposition regularizer: 
it is not a universal accuracy booster, but a mechanism for preserving a meaningful symbolic component while assigning non-symbolic effects to the residual.

\section{Conclusion}

Hybrid symbolic--neural models promise interpretable structure together with flexible residual modeling, but this promise depends on a clean division of labor between the two components. 
We showed that this division becomes fragile when the symbolic component is discovered through sparse library learning: standard $L^2$ regularization controls the magnitude of the neural residual, but not whether it relearns symbolic structure. 
We introduced \textbf{OrthoReg}, which makes complementarity explicit by penalizing empirical overlap between the neural augmentation and the symbolic library.

Our analysis shows that projection-based regularization can induce symbolic--neural overlap under sparse discovery (\cref{thm:l2_failure}), whereas OrthoReg controls empirical orthogonality violation at global minimizers (\cref{thm:penalty_bound_main}) and yields an in-library / orthogonal-complement error split (\cref{thm:error_decomp}). 
Empirically, OrthoReg improves symbolic recovery and out-of-distribution behavior in regimes of partial library mismatch, where the library captures part of the dynamics but leaves structured residual effects.

Several limitations remain. 
Our guarantees are empirical-inner-product statements for additive models with fixed differentiable libraries and continuous sparsity relaxations; derivative supervision, feature scaling, and noise sensitivity also matter in practice (\cref{app:limitations}). 
A key next step is to extend orthogonality-regularized residual learning to broader symbolic-discovery pipelines, including canonical SINDy selection--refit, genetic-programming, and transformer-based symbolic regression. 
More broadly, OrthoReg points toward hybrid scientific models in which approximate mechanistic vocabularies and flexible residuals can coexist without erasing the symbolic explanation.
\section*{Acknowledgements}
We thank Manuel Lubetzki and Alessandro Palma for helpful discussions and constructive feedback on this manuscript.

\newpage

\bibliographystyle{plainnat}
\bibliography{bibliography}

\newpage
\appendix
\appendix

\section{Why $L^2$ Regularization Fails}
\label{app:l2_failure}

Here we analyze when and why $L^2$ regularization guarantees orthogonality. We show that APHYNITY's approach works for fixed symbolic components (\cref{prop:fixed_symbolic}), but fails when symbolic structure is learned jointly with sparsity constraints (\cref{thm:l2_failure}).

\paragraph{Problem Setup.}

Throughout this appendix, $\mathcal{F}$ is a linear function space (closed under sums and scalar multiplication) sufficiently expressive to contain $f-f_{\mathrm{phy}}$ for the $f_{\mathrm{phy}}\in\mathcal{F}_{\mathrm{phy}}$ at hand; 
this is what allows the projection arguments below.
Consider learning a hybrid model $\hat{f} = \hat{f}_{\mathrm{phy}} + \hat{f}_{\mathrm{aug}}$ where:

\begin{itemize}
    \item $\hat{f}_{\mathrm{phy}}(x;w) = \sum_{j=1}^M w_j \phi_j(x)$ with $w \in \mathbb{R}^M$ (symbolic component).
    \item $\hat{f}_{\mathrm{aug}}(x; \vartheta)$ is a neural network (neural component); 
    the family $\{\hat f_{\mathrm{aug}}(\cdot;\vartheta)\}_\vartheta$ admits a parameter $\vartheta_0$ for which $\hat f_{\mathrm{aug}}(\cdot;\vartheta_0)\equiv 0$ (e.g.\ all final-layer weights set to zero for an MLP without an additive skip connection at the output).
\end{itemize}

The standard $L^2$-regularized loss is:

\begin{equation}
\label{eq:l2_loss}
\mathcal{L}_{\mathrm{L^2}}(w, \vartheta) = \underbrace{\|f - \hat{f}_{\mathrm{phy}} - \hat{f}_{\mathrm{aug}}\|_{\mathcal{D}}^2}_{\text{reconstruction}} + \underbrace{\mu \|w\|_1}_{\text{sparsity}} + \underbrace{\lambda_2 \|\hat{f}_{\mathrm{aug}}\|_{\mathcal{D}}^2}_{\text{$L^2$ regularization}}
\end{equation}

\paragraph{When $L^2$ Regularization Guarantees Orthogonality.}

We first establish the positive result: $L^2$ regularization works as intended when the symbolic component is fixed. This is a key result of~\citep{yin2021augmenting}.

\begin{proposition}[Fixed symbolic hypothesis class]
\label{prop:fixed_symbolic}
Assume $\mathcal{F}_{\mathrm{phy}} = \operatorname{span}\{\phi_1,\ldots,\phi_M\}$ is fixed, $f_{\mathrm{aug}}$ is optimized over $\mathcal{F}$, and $\lambda_2>0$. Consider
\begin{equation*}
\min_{f_{\mathrm{phy}} \in \mathcal{F}_{\mathrm{phy}},\, f_{\mathrm{aug}}\in\mathcal{F}}
\|f - f_{\mathrm{phy}} - f_{\mathrm{aug}}\|_{\mathcal{D}}^2 + \lambda_2 \|f_{\mathrm{aug}}\|_{\mathcal{D}}^2.
\end{equation*}
Then any minimizer satisfies $f_{\mathrm{phy}} = P_{\mathcal{F}_{\mathrm{phy}}}^{\mathcal{D}}(f)$ and $f_{\mathrm{aug}} = \tfrac{1}{1+\lambda_2}(f - f_{\mathrm{phy}})$, hence $f_{\mathrm{aug}} \perp \mathcal{F}_{\mathrm{phy}}$.
\end{proposition}

\begin{proof}
For fixed $f_{\mathrm{phy}}$, the objective is quadratic in $f_{\mathrm{aug}}$ and is minimized, up to $\mathcal D$-equivalence, by $f_{\mathrm{aug}}=\tfrac{1}{1+\lambda_2}(f-f_{\mathrm{phy}})$.
Plugging back reduces the problem to
\begin{equation*}
\min_{f_{\mathrm{phy}}\in\mathcal{F}_{\mathrm{phy}}}\|f - f_{\mathrm{phy}}\|_{\mathcal{D}}^2,
\end{equation*}
whose minimizer is the orthogonal projection $P_{\mathcal{F}_{\mathrm{phy}}}^{\mathcal{D}}(f)$. The residual $f - f_{\mathrm{phy}}$ is orthogonal to $\mathcal{F}_{\mathrm{phy}}$, and scaling preserves orthogonality.
\end{proof}

\paragraph{When $L^2$ Regularization Fails.}

The situation changes when we learn the symbolic structure itself. 
The following theorem shows that joint optimization breaks the orthogonality guarantee.

\begin{theorem}[$L^2$ failure with sparse symbolic learning]
\label{thm:l2_failure}
There exist $f$ and a fixed symbolic library $\{\phi_j\}_{j=1}^M$ such that the minimizer of \cref{eq:l2_loss} is \emph{not} orthogonal, that is, $\hat{f}_{\mathrm{aug}}^* \not\perp \mathcal{F}_{\mathrm{phy}}$, even though $f \in \mathcal{F}_{\mathrm{phy}}$.
\end{theorem}

\begin{proof}

Let $\phi_1,\phi_2$ be orthonormal w.r.t.\ $\langle\cdot,\cdot\rangle_{\mathcal{D}}$ and set $f=\phi_1+\phi_2$, so $f \in \mathcal{F}_{\mathrm{phy}}:=\operatorname{span}\{\phi_1,\phi_2\}$. Consider the (function-space) version of \cref{eq:l2_loss}, with $g$ standing in for $\hat f_{\mathrm{aug}}$ to lighten notation:
\begin{equation*}
\min_{w\in\mathbb{R}^2,\,g\in\mathcal{F}}\ \|f-(w_1\phi_1+w_2\phi_2+g)\|_{\mathcal{D}}^2+\mu\|w\|_1+\lambda_2\|g\|_{\mathcal{D}}^2.
\end{equation*}
For fixed $w$, the unique minimizer in $g$ is $g(w)=\tfrac{1}{1+\lambda_2}(f-w_1\phi_1-w_2\phi_2)$, and the objective reduces to the Lasso problem
\begin{equation*}
\min_{w\in\mathbb{R}^2}\ \tfrac{\lambda_2}{1+\lambda_2}\,\|f-(w_1\phi_1+w_2\phi_2)\|_{\mathcal{D}}^2+\mu\|w\|_1.
\end{equation*}
Since $\phi_1,\phi_2$ are orthonormal and $f=\phi_1+\phi_2$, the problem decouples and for $0<\mu<\tfrac{2\lambda_2}{1+\lambda_2}$ the unique minimizer is
\begin{equation*}
w_1^*=w_2^*=1-\tfrac{\mu(1+\lambda_2)}{2\lambda_2}.
\end{equation*}
Thus
\begin{equation*}
g^* = g(w^*) = \tfrac{\mu}{2\lambda_2}(\phi_1+\phi_2),
\end{equation*}
so $\langle g^*,\phi_j\rangle_{\mathcal{D}}=\tfrac{\mu}{2\lambda_2}\neq 0$ for $j\in\{1,2\}$. Hence $g^*\not\perp \mathcal{F}_{\mathrm{phy}}$ although $f\in\mathcal{F}_{\mathrm{phy}}$.
\end{proof}


\section{Theoretical Guarantees of OrthoReg}
\label{app:ortho_theory}

Having established that $L^2$ regularization can fail under joint optimization with sparsity constraints, we now develop the theory of orthogonal regularization. We show that explicit orthogonality penalties yield quantitative control of neural--symbolic overlap and provide useful geometric structure.

\paragraph{The Orthogonality Constraint.}

OrthoReg solves:
\begin{equation}
\label{eq:ortho_loss}
\min_{w, \vartheta}\ \mathcal{L}_{\mathrm{fit}}(w,\vartheta) + \mu \|w\|_1 + \lambda \sum_{j=1}^M \langle \hat{f}_{\mathrm{aug}}(\cdot;\vartheta), \phi_j \rangle_{\mathcal{D}}^2,
\end{equation}
where $\mathcal{L}_{\mathrm{fit}}$ is any non-negative data-fit term (e.g., \cref{eq:fit_loss_vf,eq:fit_loss_step}).
Unlike $L^2$ regularization, OrthoReg penalizes the squared inner products
$\langle \hat{f}_{\mathrm{aug}}, \phi_j \rangle_{\mathcal{D}}^2$
rather than the total norm $\|\hat{f}_{\mathrm{aug}}\|_{\mathcal{D}}^2$.

\paragraph{Penalty method: orthogonality at optimum.}

\begin{theorem}[Penalty bound at optimum]
\label{thm:ortho_stationary}
Assume $\lambda>0$. Let $(w^*,\vartheta^*)$ be a global minimizer of \cref{eq:ortho_loss} and let
\begin{equation*}
w_{\mathrm{sym}} \in \arg\min_{w}\ \mathcal{L}_{\mathrm{fit}}(w,\vartheta_0) + \mu\|w\|_1
\end{equation*}
be the best pure symbolic fit, where $\hat f_{\mathrm{aug}}(\cdot;\vartheta_0)\equiv 0$. Then
\begin{equation*}
\sum_{j=1}^M \langle \hat{f}_{\mathrm{aug}}(\cdot;\vartheta^*), \phi_j \rangle_{\mathcal{D}}^2
\le
\frac{1}{\lambda}\Big(\mathcal{L}_{\mathrm{fit}}(w_{\mathrm{sym}},\vartheta_0)+\mu\|w_{\mathrm{sym}}\|_1\Big).
\end{equation*}
In particular, for any sequence of global minimizers with $\lambda\to\infty$, the empirical orthogonality violations converge to zero.
\end{theorem}

\begin{proof}
By optimality of $(w^*,\vartheta^*)$, we have
\begin{equation*}
\begin{aligned}
&\mathcal{L}_{\mathrm{fit}}(w^*,\vartheta^*)
+\mu\|w^*\|_1
+\lambda\sum_{j=1}^M
\bigl\langle \hat{f}_{\mathrm{aug}}(\cdot;\vartheta^*), \phi_j \bigr\rangle_{\mathcal{D}}^2 \\
&\leqslant\;
\mathcal{L}_{\mathrm{fit}}(w_{\mathrm{sym}},\vartheta_0)
+\mu\|w_{\mathrm{sym}}\|_1,
\end{aligned}
\end{equation*}

where $\hat f_{\mathrm{aug}}(\cdot;\vartheta_0)\equiv 0$ by definition. Dropping the non-negative first two terms on the left yields the stated bound. The final claim follows by sending $\lambda\to\infty$.
\end{proof}

\paragraph{Remark (uniqueness up to $\mathcal D$-equivalence).}
The empirical $\langle\cdot,\cdot\rangle_{\mathcal D}$ is a seminorm on $\mathcal F$: 
functions agreeing on $\mathcal D$ are identified. 
The minimisers in \cref{prop:fixed_symbolic} and \cref{thm:penalty_bound_main} are therefore unique only up to $\mathcal D$-equivalence; 
this matters most when $\mathcal D$ does not separate the augmentation hypothesis class on $\mathcal X$, and is mild in our reported regimes where the empirical samples are dense.

\paragraph{Approximation Quality.}

Beyond enforcing orthogonality, we want to understand what orthogonality buys in terms of approximation structure. 
We decompose the error into symbolic and orthogonal-complement components with the orthogonal decomposition:

\begin{theorem}[Orthogonal error decomposition]
\label{thm:error_decomp}
Let $\mathcal{F}_{\mathrm{phy}}$ be a closed linear subspace of $(\mathcal{F},\langle\cdot,\cdot\rangle_{\mathcal{D}})$ and let $P_{\mathcal{F}_{\mathrm{phy}}}^{\mathcal{D}}$ denote the orthogonal projection onto $\mathcal{F}_{\mathrm{phy}}$ with respect to $\langle\cdot,\cdot\rangle_{\mathcal{D}}$.
For any $f_{\mathrm{phy}}\in\mathcal{F}_{\mathrm{phy}}$ and any $f_{\mathrm{aug}}\in\mathcal{F}$ satisfying $f_{\mathrm{aug}}\perp\mathcal{F}_{\mathrm{phy}}$, we have

\begin{align*}
\|f-(f_{\mathrm{phy}}+f_{\mathrm{aug}})\|_{\mathcal{D}}^2
&=
\|P_{\mathcal{F}_{\mathrm{phy}}}^{\mathcal{D}}(f)-f_{\mathrm{phy}}\|_{\mathcal{D}}^2 \\
&\quad +
\|f-P_{\mathcal{F}_{\mathrm{phy}}}^{\mathcal{D}}(f)-f_{\mathrm{aug}}\|_{\mathcal{D}}^2 .
\end{align*}

\end{theorem}

\begin{proof}
Write $f=P_{\mathcal{F}_{\mathrm{phy}}}^{\mathcal{D}}(f) + (f-P_{\mathcal{F}_{\mathrm{phy}}}^{\mathcal{D}}(f))$ with $P_{\mathcal{F}_{\mathrm{phy}}}^{\mathcal{D}}(f)\in\mathcal{F}_{\mathrm{phy}}$ and $f-P_{\mathcal{F}_{\mathrm{phy}}}^{\mathcal{D}}(f)\perp\mathcal{F}_{\mathrm{phy}}$.
Then
\begin{align*}
f-(f_{\mathrm{phy}}+f_{\mathrm{aug}})
&= 
\big(P_{\mathcal{F}_{\mathrm{phy}}}^{\mathcal{D}}(f)-f_{\mathrm{phy}}\big) \\
&\quad +
\big(f-P_{\mathcal{F}_{\mathrm{phy}}}^{\mathcal{D}}(f)-f_{\mathrm{aug}}\big),
\end{align*}
where the first term lies in $\mathcal{F}_{\mathrm{phy}}$ and the second term lies in $\mathcal{F}_{\mathrm{phy}}^{\perp}$, hence they are orthogonal. Expanding the squared norm yields the identity.
\end{proof}


\subsection{Finite-Sample Analysis}
\label{app:finite_sample}

Our theoretical analysis so far has operated under the empirical inner product.
In finite-sample settings we enforce orthogonality empirically, and here we quantify how well empirical inner products approximate their population counterparts.

\paragraph{Empirical vs Population Orthogonality.}

\begin{definition}[Orthogonality Gap]
Define the empirical inner product
\begin{equation*}
\langle f, g \rangle_{\mathcal{D}}
= \frac{1}{N}\sum_{i=1}^N f(x_i)^\top g(x_i),
\end{equation*}
and the population inner product
\begin{equation*}
\langle f, g \rangle_{\nu}
= \mathbb{E}_{x \sim \nu}[f(x)^\top g(x)].
\end{equation*}
The orthogonality gap is $|\langle f, g \rangle_{\mathcal{D}} - \langle f, g \rangle_{\nu}|$.
\end{definition}

The following concentration result applies to fixed functions $f$ and $g$, that is, functions that do not depend on the specific sample $\mathcal{D}$.

\begin{theorem}[Finite-Sample Guarantee]
\label{thm:finite_sample}
Let $f, g: \mathcal{X} \to \mathbb{R}^n$ be fixed functions and assume $|f(x)^\top g(x)| \leq R$ almost surely. Let $\mathcal{D} = \{x_1, \ldots, x_N\}$ be i.i.d.\ samples from distribution $\nu$. Then with probability at least $1-\delta$:
\begin{equation*}
|\langle f, g \rangle_{\mathcal{D}} - \langle f, g \rangle_{\nu}|
\leq R\sqrt{\frac{2\log(2/\delta)}{N}}.
\end{equation*}
\end{theorem}

\begin{proof}
Define random variables $Z_i = f(x_i)^\top g(x_i)$ for $i=1,\ldots,N$, where $x_i \sim \nu$ are i.i.d.\ and satisfy $Z_i \in [-R, R]$ almost surely. Then
\begin{equation*}
\langle f, g \rangle_{\mathcal{D}} = \frac{1}{N}\sum_{i=1}^N Z_i,
\qquad
\langle f, g \rangle_{\nu} = \mathbb{E}[Z_i].
\end{equation*}

Hoeffding's inequality states that for i.i.d.\ random variables $Z_i \in [a,b]$,
\begin{equation*}
\mathbb{P}\!\left(\left|\frac{1}{N}\sum_{i=1}^N Z_i - \mathbb{E}[Z_i]\right|
\geq \epsilon \right)
\leq 2\exp\!\left(-\frac{2N\epsilon^2}{(b-a)^2}\right).
\end{equation*}

Applying this with $a=-R$, $b=R$ and solving
\begin{equation*}
2\exp(-2N\epsilon^2 / 4R^2) = \delta
\end{equation*}
yields
\begin{equation*}
\epsilon = R\sqrt{2\log(2/\delta)/N}.
\end{equation*}
\end{proof}

\begin{corollary}[Population Orthogonality]
If we enforce $\langle \hat f_{\mathrm{aug}}, \phi_j \rangle_{\mathcal{D}} = 0$ empirically, then under the assumption that $\hat f_{\mathrm{aug}}$ is fixed independently of $\mathcal{D}$, we have
\begin{equation*}
|\langle \hat f_{\mathrm{aug}}, \phi_j \rangle_{\nu}|
\leq R\sqrt{\frac{2\log(2/\delta)}{N}}
\end{equation*}
with probability $1-\delta$.
\end{corollary}

\paragraph{Intuition.}
Each inner product 
\begin{equation*}
\langle \hat{f}_{\mathrm{aug}}, \phi_j \rangle_{\mathcal{D}}
= \frac{1}{N} \sum_{i=1}^N \hat{f}_{\mathrm{aug}}(x_i)^{\top} \phi_j(x_i)
\end{equation*}
is an average of $N$ bounded random variables. Hoeffding's inequality for 
bounded variables $Z_i \in [-R, R]$ yields concentration around the population 
mean at rate $O(1/\sqrt{N})$.

\paragraph{Remark.}
Since $\hat f_{\mathrm{aug}}$ is trained on the same sample used to evaluate $\langle \hat f_{\mathrm{aug}},\phi_j\rangle_{\mathcal{D}}$, the independence assumption is violated.
A population-level guarantee for the \emph{learned} augmentation therefore requires uniform convergence over the augmentation hypothesis class (e.g.\ Rademacher or VC-style bounds), and not single-function concentration; we discuss this in \cref{app:limitations}.

\subsection{Monte Carlo Approximation}
\label{app:mc_analysis}

In training we estimate the empirical inner products in the orthogonality penalty using minibatches. We quantify the resulting stochastic approximation error.

\paragraph{Minibatch estimator.}

Let $\mathcal{D}=\{x_i\}_{i=1}^N$ and fix a constraint $\phi_j$.
For a minibatch $\mathcal{B}\subseteq\{1,\dots,N\}$ of size $B$, define
\begin{align*}
\langle \hat f_{\mathrm{aug}}, \phi_j \rangle_{\mathcal{B}}
&:= \frac{1}{B}\sum_{i\in\mathcal{B}} \hat f_{\mathrm{aug}}(x_i)^\top \phi_j(x_i),
 \\
\langle \hat f_{\mathrm{aug}}, \phi_j \rangle_{\mathcal{D}}
&:= \frac{1}{N}\sum_{i=1}^N \hat f_{\mathrm{aug}}(x_i)^\top \phi_j(x_i).
\end{align*}
Assume $|\hat f_{\mathrm{aug}}(x)^\top \phi_j(x)|\le R$ for all $x\in\mathcal{D}$.

\begin{lemma}[Minibatch deviation]
\label{lem:minibatch_deviation}
If $\mathcal{B}$ is sampled uniformly (with replacement) from $\mathcal{D}$, then for any $\delta\in(0,1)$,
\begin{equation*}
\mathbb{P}\!\left(
\bigl|\langle \hat f_{\mathrm{aug}}, \phi_j \rangle_{\mathcal{B}}
-\langle \hat f_{\mathrm{aug}}, \phi_j \rangle_{\mathcal{D}}\bigr|
\ge \epsilon\right)
\le
2\exp\!\left(-\frac{B\epsilon^2}{2R^2}\right),
\end{equation*}
and equivalently, with probability at least $1-\delta$,
\begin{equation*}
\bigl|\langle \hat f_{\mathrm{aug}}, \phi_j \rangle_{\mathcal{B}}
-\langle \hat f_{\mathrm{aug}}, \phi_j \rangle_{\mathcal{D}}\bigr|
\le
R\sqrt{\frac{2\log(2/\delta)}{B}}.
\end{equation*}
The same scaling holds for sampling without replacement (up to a finite-population correction factor).
\end{lemma}

\paragraph{Implication for optimization.}
\Cref{lem:minibatch_deviation} shows that the minibatch inner products concentrate around their full-data counterparts at rate $\mathcal{O}(B^{-1/2})$. 
In the default configuration the orthogonality penalty is evaluated on the full training set each step (all trajectories fit into one batch). 
For completeness, we record a standard concentration bound for minibatch estimation.

\section{When Residual Pipelines Can Reintroduce Symbolic Overlap}
\label{app:sequential_extension}

We revisit the setting of \cref{ex:intro_sparsity_overlap}. Let
\(\phi_1,\phi_2\in\mathcal F\) be orthonormal with respect to
\(\langle\cdot,\cdot\rangle_{\mathcal D}\), let
\(f=\phi_1+\phi_2\), and let the Stage-1 symbolic model be fitted with an
\(L^1\) penalty of weight \(\mu\in(0,2)\). The orthonormal Lasso solution is
\[
\hat w_j = 1-\mu/2,\qquad j=1,2,
\]
and therefore leaves the empirical residual
\[
r
:= f-\sum_{j=1}^2 \hat w_j\phi_j
= \frac{\mu}{2}(\phi_1+\phi_2)
\in \operatorname{span}\{\phi_1,\phi_2\}.
\]
Thus, the residue created by sparsity regularization lies exactly in the
symbolic span. The question is whether a neural residual fitted after, or
through, the symbolic component is forced to remain orthogonal to this span.

\begin{proposition}[Unconstrained residual learning and compositional pipelines can retain symbolic overlap]
\label{prop:seq_two_stage_failure}
In the setting above, the following hold.
\begin{enumerate}[leftmargin=*]
\item \textbf{Unconstrained residual learning.}
If \(w\) is fitted by Lasso and an unconstrained augmentation \(g\in\mathcal F\)
is subsequently fitted to the residual, then
\[
\hat g=r,
\qquad
\langle \hat g,\phi_j\rangle_{\mathcal D}=\mu/2,
\qquad j=1,2.
\]

\item \textbf{Compositional architecture with an input-space penalty.}
Consider a compositional architecture
\[
\hat f(x) = (I+N_\vartheta)(T_w(x))
          = T_w(x)+N_\vartheta(T_w(x)),
\qquad
T_w(x)=\sum_{j=1}^2 w_j\phi_j(x).
\]
For this simplified example, $T_w$ denotes the symbolic component itself; 
in compositional models with a residual structure $\hat f = T_w + N_\vartheta\circ T_w$, $T_w$ is the physics map onto which $N_\vartheta$ composes.
Define the effective additive augmentation \emph{at the Lasso solution} $\hat w$
\[
\widetilde N_\vartheta(x):=N_\vartheta(T_{\hat w}(x)).
\]
Assume that, at the Lasso solution \(\hat w_j=1-\mu/2\), the \(2N\) points
\[
\{x_i\}_{i=1}^N \cup \{T_{\hat w}(x_i)\}_{i=1}^N
\]
are all distinct. Then there exists a sufficiently expressive \(N_\vartheta\)
such that
\[
\hat f(x_i)=f(x_i),\qquad
N_\vartheta(x_i)=0,
\qquad
\widetilde N_\vartheta(x_i)=r(x_i)
\]
for all \(i=1,\ldots,N\). Consequently,
\[
\sum_i N_\vartheta(x_i)^\top \phi_j(x_i)=0,\qquad j=1,2,\quad\text{but}\quad
\langle \widetilde N_\vartheta,\phi_j\rangle_{\mathcal D}
=
\langle r,\phi_j\rangle_{\mathcal D}
=
\mu/2.
\]
Thus, an OrthoReg-style penalty applied to \(N_\vartheta\) at the original
inputs \(\{x_i\}\) does not control the additive overlap of the effective
augmentation \(N_\vartheta\circ T_{\hat w}\).
\end{enumerate}
\end{proposition}

\begin{proof}
For unconstrained residual learning, orthonormality gives the usual soft-thresholding solution \(\hat w_j=(1-\mu/2)_+\). 
Since \(\mu\in(0,2)\), this gives \(\hat w_j=1-\mu/2\), and an unconstrained augmentation fits the residual exactly:
\[
\hat g=f-\sum_{j=1}^2\hat w_j\phi_j
=\frac{\mu}{2}(\phi_1+\phi_2).
\]
Taking empirical inner products with \(\phi_1\) and \(\phi_2\) gives
\(\langle\hat g,\phi_j\rangle_{\mathcal D}=\mu/2\).

For the compositional case, exact fit at the Lasso coefficients requires
\[
N_\vartheta(T_{\hat w}(x_i))
=
f(x_i)-T_{\hat w}(x_i)
=
r(x_i).
\]
This constrains \(N_\vartheta\) at the transported points \(T_{\hat w}(x_i)\),
not at the original data points \(x_i\). Since the \(2N\) points are distinct,
a sufficiently expressive \(N_\vartheta\) can interpolate both
\[
N_\vartheta(T_{\hat w}(x_i))=r(x_i)
\quad\text{and}\quad
N_\vartheta(x_i)=0.
\]
The empirical penalty on \(N_\vartheta(x_i)\) is then zero, while the effective
additive augmentation satisfies
\[
\widetilde N_\vartheta(x_i)
=
N_\vartheta(T_{\hat w}(x_i))
=
r(x_i),
\]
and hence
\(\langle\widetilde N_\vartheta,\phi_j\rangle_{\mathcal D}=\mu/2\).
\end{proof}

\paragraph{Orthogonality-regularised residual learning.}
The failure above is not caused by staging itself, but by fitting the residual without an orthogonality constraint. If Stage 2 instead minimises
\[
\|r-g\|_{\mathcal D}^2+\lambda\sum_j\langle g,\phi_j\rangle_{\mathcal D}^2,
\]
then, on the orthonormal toy, $\langle\hat g,\phi_j\rangle_{\mathcal D} = (\mu/2)/(1+\lambda) \to 0$ as $\lambda\to\infty$. 
Orthogonalisation therefore suppresses symbolic--neural overlap whether imposed jointly or in an explicitly regularised residual stage; 
it does not remove the $L^1$ shrinkage bias on $w$, so the shrinkage residue is left as fit error rather than absorbed by the neural component.

\section{Computational Complexity}
\label{app:complexity}

We briefly discuss the additional computational overhead introduced by the orthogonality penalty.

\paragraph{Per-iteration cost.}
Let $B$ be the batch size, $n$ the state dimension, and $M$ the number of library terms.
For the vector-field regression loss used in the main experiments, computing the penalty requires evaluating $\phi_j(x_i)$ and forming inner products $\hat f_{\mathrm{aug}}(x_i)^\top \phi_j(x_i)$ for all $i\in[B]$ and $j\in[M]$, which costs $\mathcal{O}(B\,n\,M)$ arithmetic operations in addition to the usual network forward/backward passes; 
backpropagation has the same order.
Rollout or sequence-prediction losses multiply this cost by the number of evaluated transitions or time steps.

\paragraph{Timing protocol.}
Wall-clock per epoch is measured as the median over a fixed window of training epochs after a short warm-up, on the modified damped pendulum ($B=1000$, $n=2$, $M=20$), with identical optimizer state and data loaders across configurations.
The reported $5$--$15\%$ overhead is the relative increase of the OrthoReg median over the Hybrid $L^2$ median across five seeds, isolating the cost of the orthogonality term $g_j$; overhead grows linearly in $M$ at fixed $(B,n)$.

\paragraph{Memory requirements.}
The additional memory is dominated by storing the evaluated library features on the batch, which scales as $\mathcal{O}(B\,n\,M)$ floats.
For example, with $B=1000$, $n=2$, and $M=20$, this corresponds to approximately $4{\times}10^4$ scalars (about $160$\,KB in float32); on the SIR system ($M=45$) the same calculation gives roughly $360$\,KB, well below the network activations cached for backpropagation.

\paragraph{Mini-batch gradient bias.}
The penalty $\sum_j\langle\hat f_{\mathrm{aug}},\phi_j\rangle_{\mathcal D}^2$ is a squared mean of bounded random variables, so the standard one-batch unbiased Monte-Carlo estimator of $\nabla_\vartheta\,\langle\hat f_{\mathrm{aug}},\phi_j\rangle_{\mathcal D}$ (which would multiply two independent mini-batch estimates) is replaced in our implementation by the squared single-batch average. Under stochastic gradient descent this introduces an additive bias of order $\mathcal O(1/B)$ relative to the full-data quantity (exactly the variance of the per-batch inner-product estimate); the full-batch evaluation used in the main experiments ($B$ equal to the full training set, see \cref{app:exp_design}) makes this bias zero by construction, and the mini-batch ablation in \cref{app:mc_sampling_abl} confirms that the ranking and quantitative trends carry over for $B\ge 200$. A two-batch unbiased estimator is a drop-in alternative for memory-constrained settings.

\paragraph{Hardware and total compute.}
All reported experiments were run on an internal HPC cluster, allocating one general-purpose GPU per job, $\approx 64$\,GB of RAM, and a small CPU allocation (typically $\le 8$ cores). The workloads are intentionally small---autonomous ODEs in $n\le 4$ states, MLPs of width $64$--$128$, $2000$ Adam epochs, and batch sizes $\le 1000$---so that OrthoReg remains CPU-tractable; we used GPUs primarily to accelerate the neural ODE component, which dominates wall-clock time.
Each individual training run completes within a small number of GPU-minutes; the full reported grid (main systems, ablations, and cross-system validation, each repeated over five seeds) accounts for the bulk of the project's compute, and preliminary tuning and discarded runs add a similar order-of-magnitude overhead on top.

\newpage
\section{Additional Ablations}
\label{app:additional_experiments}

\subsection{Full Dataset Difficulty Ablation}
\label{app:dataset_diff}

\begin{table}[h]
\centering

\small
\setlength{\tabcolsep}{5pt}  
\caption{
Dataset difficulty ablation across missing dynamics regimes (mean effect strength absent from the symbolic library). 
OrthoReg's gains are largest at medium difficulty; in the low-difficulty regime it trades ID and OOD-T2 accuracy for higher F1 and OOD-T3 generalization.}
\label{tab:difficulty_ablation_streamlined}
\begin{tabular}{lccc}
\toprule
\textbf{Metric} & \textbf{Pure} & \textbf{$L^2$} & \textbf{OrthoReg} \\
\midrule
\multicolumn{4}{c}{\textbf{Low Difficulty ($\beta_1=0.1, \beta_2=0.08, \beta_3=0.05$)}} \\
\midrule
MSE(${\scriptstyle \dot{x}_\mathrm{ID},\ \mathrm{e}{-4}}$) {\scalebox{0.8}{($\downarrow$)}}
  & $\bm{7.77 {\scalebox{0.8}{$\pm 1.64$}}}$
  & $10.99 {\scalebox{0.8}{$\pm 0.83$}}$
  & $26.82 {\scalebox{0.8}{$\pm 5.00$}}$ \\
MSE(${\scriptstyle \dot{x}_\mathrm{OOD,T2},\ \mathrm{e}{-3}}$) {\scalebox{0.8}{($\downarrow$)}}
  & $\bm{1.53 {\scalebox{0.8}{$\pm 0.26$}}}$
  & $2.01 {\scalebox{0.8}{$\pm 0.00$}}$
  & $2.79 {\scalebox{0.8}{$\pm 0.09$}}$ \\
MSE(${\scriptstyle \dot{x}_\mathrm{OOD,T3},\ \mathrm{e}{0}}$) {\scalebox{0.8}{($\downarrow$)}}
  & $161.60 {\scalebox{0.8}{$\pm 166.18$}}$
  & $5.86 {\scalebox{0.8}{$\pm 0.01$}}$
  & $\bm{0.58 {\scalebox{0.8}{$\pm 0.06$}}}$ \\
F1 {\scalebox{0.8}{($\uparrow$)}}
  & $0.50 {\scalebox{0.8}{$\pm 0.09$}}$
  & $0.72 {\scalebox{0.8}{$\pm 0.05$}}$
  & $\bm{0.86 {\scalebox{0.8}{$\pm 0.00$}}}$ \\ 
\midrule
\multicolumn{4}{c}{\textbf{Medium Difficulty ($\beta_1=0.3, \beta_2=0.25, \beta_3=0.15$)}} \\
\midrule
MSE(${\scriptstyle \dot{x}_\mathrm{ID},\ \mathrm{e}{-2}}$) {\scalebox{0.8}{($\downarrow$)}}
  & $\bm{1.27 {\scalebox{0.8}{$\pm 0.07$}}}$
  & $1.41 {\scalebox{0.8}{$\pm 0.05$}}$
  & $1.89 {\scalebox{0.8}{$\pm 0.01$}}$ \\
MSE(${\scriptstyle \dot{x}_\mathrm{OOD,T2},\ \mathrm{e}{-2}}$) {\scalebox{0.8}{($\downarrow$)}}
  & $\bm{1.15 {\scalebox{0.8}{$\pm 0.05$}}}$
  & $1.27 {\scalebox{0.8}{$\pm 0.04$}}$
  & $1.54 {\scalebox{0.8}{$\pm 0.02$}}$ \\
MSE(${\scriptstyle \dot{x}_\mathrm{OOD,T3},\ \mathrm{e}{0}}$) {\scalebox{0.8}{($\downarrow$)}}
  & $940.60 {\scalebox{0.8}{$\pm 614.71$}}$
  & $218.49 {\scalebox{0.8}{$\pm 111.32$}}$
  & $\bm{2.45 {\scalebox{0.8}{$\pm 0.43$}}}$ \\
F1 {\scalebox{0.8}{($\uparrow$)}}
  & $0.43 {\scalebox{0.8}{$\pm 0.09$}}$
  & $0.61 {\scalebox{0.8}{$\pm 0.03$}}$
  & $\bm{0.93 {\scalebox{0.8}{$\pm 0.15$}}}$ \\
\midrule
\multicolumn{4}{c}{\textbf{High Difficulty ($\beta_1=2.0, \beta_2=2.0, \beta_3=2.0$)}} \\
\midrule
MSE(${\scriptstyle \dot{x}_\mathrm{ID},\ \mathrm{e}{-2}}$) {\scalebox{0.8}{($\downarrow$)}}
  & $3.91 {\scalebox{0.8}{$\pm 0.13$}}$
  & $3.90 {\scalebox{0.8}{$\pm 0.11$}}$
  & $\bm{3.86 {\scalebox{0.8}{$\pm 0.09$}}}$ \\
MSE(${\scriptstyle \dot{x}_\mathrm{OOD,T2},\ \mathrm{e}{-2}}$) {\scalebox{0.8}{($\downarrow$)}}
  & $4.49 {\scalebox{0.8}{$\pm 0.17$}}$
  & $4.46 {\scalebox{0.8}{$\pm 0.10$}}$
  & $\bm{4.43 {\scalebox{0.8}{$\pm 0.09$}}}$ \\
MSE(${\scriptstyle \dot{x}_\mathrm{OOD,T3},\ \mathrm{e}{-1}}$) {\scalebox{0.8}{($\downarrow$)}}
  & $4.49 {\scalebox{0.8}{$\pm 2.61$}}$
  & $4.06 {\scalebox{0.8}{$\pm 0.75$}}$
  & $\bm{2.64 {\scalebox{0.8}{$\pm 0.48$}}}$ \\
F1 {\scalebox{0.8}{($\uparrow$)}}
  & $0.46 {\scalebox{0.8}{$\pm 0.05$}}$
  & $0.43 {\scalebox{0.8}{$\pm 0.04$}}$
  & $\bm{0.52 {\scalebox{0.8}{$\pm 0.07$}}}$ \\
\bottomrule
\end{tabular}
\begin{flushleft}
\footnotesize
$^{\dagger}$For Medium, OOD T3 reports derivative MSE instead of state MSE due to numerical instabilities during trajectory integration for baseline methods. Low and High report state MSE for OOD T3.
\end{flushleft}
\end{table}

\Cref{tab:difficulty_ablation_streamlined} reports the full dataset difficulty ablation across regimes of increasing missing dynamics (controlled by the mean effect strength $\beta$ of terms absent from the symbolic library; cf.~$\beta_i$ in \cref{eq:pend}). 
In the main paper, \cref{fig:ablations}a summarizes the central trends using stress-test generalization ($x_{\mathrm{OOD,T3}}$) and symbolic recovery (F1). 
The complete table additionally includes in-distribution derivative accuracy ($\dot{x}_{\mathrm{ID}}$) and a milder OOD evaluation without parameter shifts ($x_{\mathrm{OOD,T2}}$).

\paragraph{Low difficulty ($\beta_1=0.1, \beta_2=0.08, \beta_3=0.05$).}
When the true dynamics are close to the symbolic library, all methods achieve low errors on in-distribution and mild OOD evaluations (\cref{tab:difficulty_ablation_streamlined}, first block), and the hybrid component is not strongly required. 
Interestingly, the most flexible baselines can remain competitive on these simpler evaluations, consistent with the observation that regularization is least critical when the library is largely sufficient. 
Under the stress-test setting ($x_{\mathrm{OOD,T3}}$), however, OrthoReg remains substantially more robust (\cref{fig:ablations}a), indicating that enforcing a clean split can mitigate extrapolation failures even when the library mismatch is small.

\paragraph{Medium difficulty ($\beta_1=0.3, \beta_2=0.25, \beta_3=0.15$).}
This regime shows the most pronounced separation between methods.
OrthoReg simultaneously improves OOD performance and symbolic recovery, yielding markedly lower $x_{\mathrm{OOD,T3}}$ error and substantially higher F1 than both the unregularized baseline and $L^2$ regularization (\cref{fig:ablations}a; \cref{tab:difficulty_ablation_streamlined}, second block).
Conceptually, the orthogonality penalty constrains the augmentation to capture residual structure without redundantly re-learning library terms, stabilising identification and improving generalisation.

\paragraph{High difficulty ($\beta_1=2.0, \beta_2=2.0, \beta_3=2.0$).}
When the missing dynamics strongly exceed the representational capacity of the library, all methods deteriorate in symbolic recovery and OOD performance (\cref{tab:difficulty_ablation_streamlined}, third block), reflecting that the symbolic--neural split becomes intrinsically harder to recover. 
Nevertheless, OrthoReg retains a modest advantage on $x_{\mathrm{OOD,T3}}$ and F1 (\cref{fig:ablations}a). 
In this regime, once the residual dominates, orthogonality alone is insufficient to recover interpretable structure and the symbolic--neural split becomes harder to identify.

\begin{table}[h]
\centering
\small
\setlength{\tabcolsep}{5pt}  
\caption{Sampling scheme ablation (regular vs. irregular). All methods degrade under irregular sampling; OrthoReg retains a small relative advantage on F1.}

\label{tab:sampling_ablation_streamlined}
\begin{tabular}{lccc}
\toprule
\textbf{Metric} & \textbf{Pure} & \textbf{$L^2$} & \textbf{OrthoReg} \\
\midrule
\multicolumn{4}{c}{\textbf{Regular Sampling}} \\
\midrule
MSE(${\scriptstyle \dot{x}_\mathrm{ID},\ \mathrm{e}{-2}}$) {\scalebox{0.8}{($\downarrow$)}} 
  & $\bm{1.27 {\scalebox{0.8}{$\pm 0.07$}}}$ 
  & $1.41 {\scalebox{0.8}{$\pm 0.05$}}$ 
  & $1.89 {\scalebox{0.8}{$\pm 0.01$}}$ \\
MSE(${\scriptstyle x_\mathrm{OOD,T2},\ \mathrm{e}{0}}$) {\scalebox{0.8}{($\downarrow$)}} 
  & $1.02 {\scalebox{0.8}{$\pm 0.07$}}$ 
  & $1.07 {\scalebox{0.8}{$\pm 0.02$}}$ 
  & $\bm{1.01 {\scalebox{0.8}{$\pm 0.01$}}}$ \\
MSE(${\scriptstyle \dot{x}_\mathrm{OOD,T3},\ \mathrm{e}{2}}$) {\scalebox{0.8}{($\downarrow$)}} 
  & $34.00 {\scalebox{0.8}{$\pm 22.22$}}$ 
  & $7.90 {\scalebox{0.8}{$\pm 4.02$}}$ 
  & $\bm{0.02 {\scalebox{0.8}{$\pm 0.00$}}}$ \\
F1 {\scalebox{0.8}{($\uparrow$)}} 
  & $0.43 {\scalebox{0.8}{$\pm 0.09$}}$ 
  & $0.61 {\scalebox{0.8}{$\pm 0.03$}}$ 
  & $\bm{0.93 {\scalebox{0.8}{$\pm 0.15$}}}$ \\
\midrule
\multicolumn{4}{c}{\textbf{Irregular Sampling}} \\
\midrule
MSE(${\scriptstyle \dot{x}_\mathrm{ID},\ \mathrm{e}{0}}$) {\scalebox{0.8}{($\downarrow$)}} 
  & $3.52 {\scalebox{0.8}{$\pm 0.00$}}$ 
  & $3.52 {\scalebox{0.8}{$\pm 0.00$}}$ 
  & $\bm{3.52 {\scalebox{0.8}{$\pm 0.00$}}}$ \\
MSE(${\scriptstyle x_\mathrm{OOD,T2},\ \mathrm{e}{0}}$) {\scalebox{0.8}{($\downarrow$)}} 
  & $3.85 {\scalebox{0.8}{$\pm 0.00$}}$ 
  & $3.85 {\scalebox{0.8}{$\pm 0.00$}}$ 
  & $\bm{3.85 {\scalebox{0.8}{$\pm 0.00$}}}$ \\
MSE(${\scriptstyle \dot{x}_\mathrm{OOD,T3},\ \mathrm{e}{1}}$) {\scalebox{0.8}{($\downarrow$)}} 
  & $4.04 {\scalebox{0.8}{$\pm 0.09$}}$ 
  & $3.90 {\scalebox{0.8}{$\pm 0.10$}}$ 
  & $\bm{3.74 {\scalebox{0.8}{$\pm 0.05$}}}$ \\
F1 {\scalebox{0.8}{($\uparrow$)}} 
  & $0.30 {\scalebox{0.8}{$\pm 0.01$}}$ 
  & $0.30 {\scalebox{0.8}{$\pm 0.01$}}$ 
  & $\bm{0.31 {\scalebox{0.8}{$\pm 0.01$}}}$ \\
\bottomrule
\end{tabular}
\end{table}

Overall, \Cref{tab:difficulty_ablation_streamlined} complements \cref{fig:ablations}a by showing that OrthoReg's strongest gains arise when the system \emph{partially} exceeds the library: small mismatches do not demand strong regularization, whereas extreme mismatches limit recoverability for all methods.

\subsection{Full Sampling Scheme Ablation}
\label{app:sampling_ablation}

\Cref{tab:sampling_ablation_streamlined} reports the full sampling scheme ablation across regular and irregular observation settings. 
Irregular sampling substantially increases derivative and OOD prediction errors and reduces symbolic recovery across all methods, consistent with increased solver error and gradient variance under non-uniform time steps (see \cref{app:exp_design}). 
Despite this degradation, OrthoReg retains a marginal advantage in symbolic recovery (F1) compared to baseline methods. 
These results complement the trends shown in \cref{fig:ablations}b and indicate that orthogonal regularization remains effective for symbolic recovery beyond idealized sampling assumptions.

\subsection{Monte Carlo Sampling Ablation}
\label{app:mc_sampling_abl}
We investigate the impact of Monte Carlo sampling on model performance by varying the number of training samples from 100 to 5000. \Cref{fig:n_samples_ablation} shows the performance across different sample sizes for the medium missing dynamics regime ($\beta_1=0.3, \beta_2=0.25, \beta_3=0.15$).

\begin{figure}[!htbp]
\centering
\includegraphics[width=\linewidth]{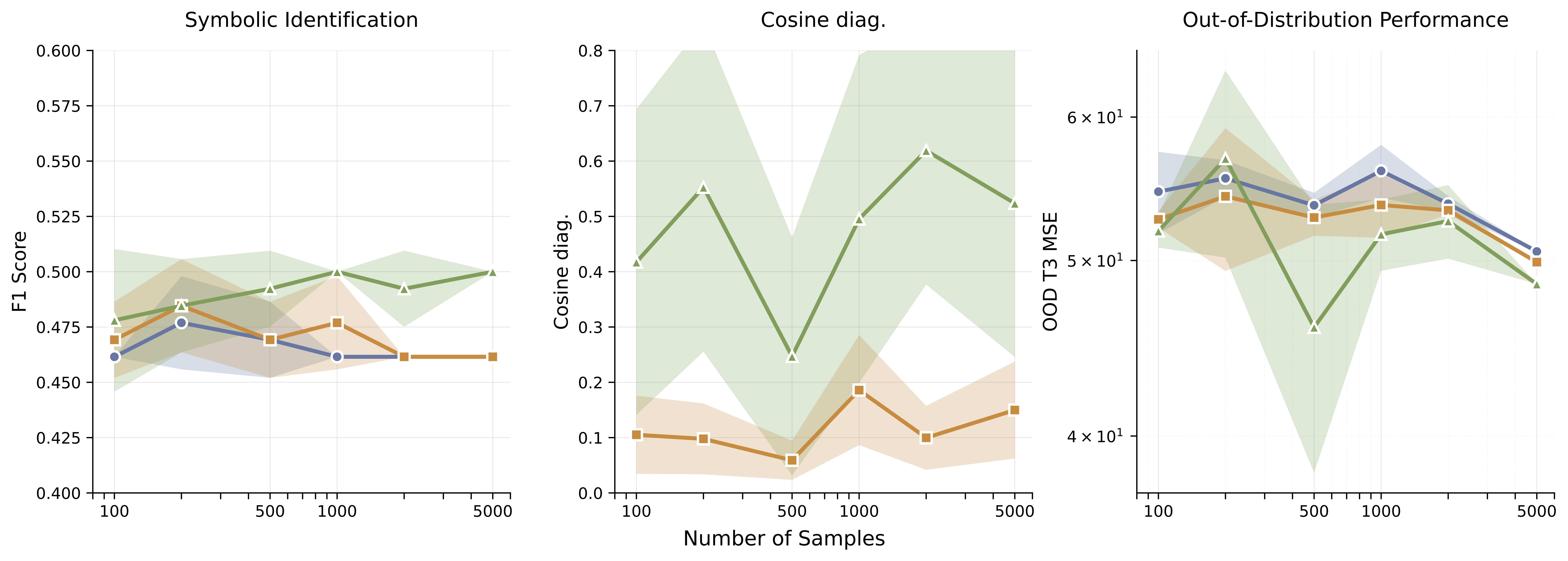}
\caption{Monte Carlo sampling ablation study (medium missing dynamics).
Performance is shown across different sample sizes (100--5000) for F1 score, $\mathrm{OOD,T2}$ MSE, and a cosine diagnostic
$\bigl|\langle \hat f_{\rm aug}, \hat f_{\rm phy}\rangle_{\mathcal D}\bigr| / (\|\hat f_{\rm aug}\|_{\mathcal D}\|\hat f_{\rm phy}\|_{\mathcal D})$
between the learned augmentation and the assembled symbolic prediction $\hat f_{\rm phy}=\sum_j \hat w_j \phi_j$ on the empirical training distribution.
The cosine takes values in $[0,1]$: smaller values mean $\hat f_{\rm aug}$ and $\hat f_{\rm phy}$ are more directionally separated under the empirical inner product, larger values mean they share direction.
We report it to visualise the concentration result of \cref{lem:minibatch_deviation}: at rate $\mathcal O(N^{-1/2})$ the empirical inner products that drive the OrthoReg penalty concentrate around their full-data values, and the cosine -- built from those same inner products -- stabilises along with them as $N$ grows.
The cosine is a within-method stability diagnostic for a single OrthoReg run; it is built from the empirical inner products that appear in the OrthoReg training penalty $\sum_j \langle \hat f_{\rm aug}, \phi_j\rangle_{\mathcal D}^2$ but is not the penalty itself.}
\label{fig:n_samples_ablation}
\end{figure}

F1 and the cosine diagnostic both stabilise as the number of training samples grows for OrthoReg, consistent with the $\mathcal O(N^{-1/2})$ concentration of the empirical inner products that the OrthoReg penalty operates on (\cref{lem:minibatch_deviation}).


\subsection{Noise Robustness Ablation}
\label{app:noise_robustness}

We investigate the robustness of OrthoReg, $L^2$ hybrid, and Pure Symbolic methods to measurement noise in the training data. Gaussian noise with standard deviation $\varepsilon \in \{0.0, 0.01, 0.05, 0.1, 0.2, 0.5\}$ is added to the state observations after trajectory generation, simulating realistic measurement uncertainty.

\begin{figure}[h]
\centering
\includegraphics[width=\linewidth]{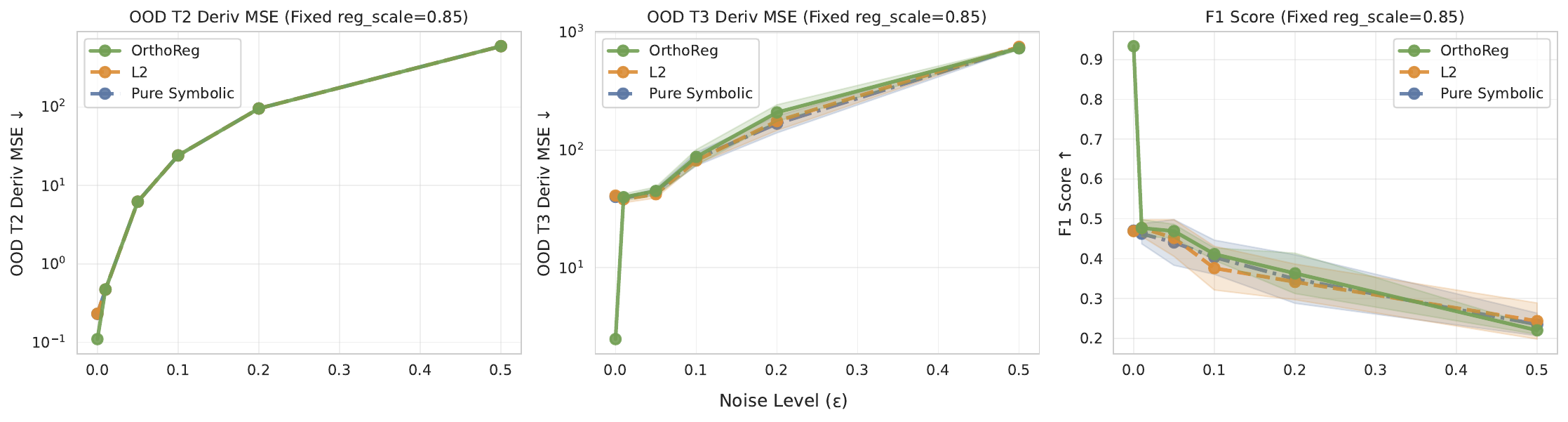}
\caption{Noise robustness study on the damped pendulum system. 
Performance degrades similarly across all methods as noise increases. 
OrthoReg maintains its advantage at zero noise (superior $\mathrm{OOD,T2}$ and $\mathrm{OOD,T3}$ performance, higher F1 score) but does not provide special robustness to noise compared to baselines. 
All methods converge to similar performance levels at high noise ($\varepsilon \geq 0.1$), indicating that measurement noise fundamentally limits symbolic recovery regardless of regularization approach.
Curves correspond, in order, to OrthoReg, $L^2$-regularised hybrid, and Pure Symbolic.
}
\label{fig:noise_robustness}
\end{figure}

\Cref{fig:noise_robustness} shows that all three methods suffer similarly from increasing noise levels. 
OrthoReg maintains its performance advantages at zero noise (lower $\mathrm{OOD,T2}$ and $\mathrm{OOD,T3}$ derivative MSE, higher F1 score), but these advantages diminish rapidly as noise increases. 
At noise levels $\varepsilon \geq 0.1$, all methods converge to similar performance, with $\mathrm{OOD,T2}$ derivative MSE increasing from $\sim 0.05$ (zero noise) to $\sim 0.6$ (high noise) and F1 scores dropping from $\sim 0.9$ to $\sim 0.2-0.4$ across all methods.

Overall, orthogonal regularization does not materially improve robustness to measurement noise in this setting; performance converges across methods at higher noise levels.
Improving robustness likely requires changes to the estimation procedure (e.g., denoising or noise-aware differentiation) rather than additional separation constraints.
At high noise ($\varepsilon\gtrsim 0.1$), the smoothed-finite-difference targets used to compute $y_i\approx\dot x(t_i)$ are themselves dominated by amplified noise (Savitzky--Golay window of length~5), so the convergence of all methods at this regime partly reflects target corruption rather than failure of the regularizer itself; an analytic-derivative ablation would disentangle the two and is left to future work (\cref{app:limitations}).


\section{Cross-System Validation}
\label{app:cross_system_validation}

The main experiments (\cref{sec:experiments}) evaluate OrthoReg on the modified pendulum and Duffing systems.
Here we assess transfer to additional systems that introduce different sources of library mismatch, focusing on Lotka--Volterra (temporal modulation) and SIR (state dependence and memory).
Together, these systems test three distinct mismatch regimes: structured feature mismatch (Duffing), non-autonomous perturbations (Lotka--Volterra), and strongly misspecified state-dependent residual structure (SIR).

\paragraph{Autonomous augmentation with non-autonomous data.}
\label{app:non_autonomy_caveat}
Both LV and SIR include explicit time-dependent factors in the ground-truth dynamics, while the augmentation $\hat f_{\rm aug}(x;\vartheta)$ in \cref{eq:orthoreg_objective} takes only the state $x$ as input.
An autonomous augmentation can therefore fit such factors only as a state-dependent envelope on the training horizon $[0,T]$; 
the OOD splits below test state extrapolation rather than time extrapolation.
Concatenating $t$ with $x$ in the network input is a one-line architectural change that would close this gap.

\subsection{Lotka--Volterra System: Temporal Coupling}
\label{app:subsec_lv}
We evaluate OrthoReg on a modified predator-prey system with temporally modulated and state-dependent interactions. The dynamics are:

\begin{align}
\frac{dx}{dt} &= \alpha x - \beta xy 
                 + \varepsilon_1 x \, 
                   \bigl(\sin(\omega_{\rm fast} t) \cos(\omega_{\rm fast} xy) \bigr) \notag 
                \times \sin\bigl(\omega_{\rm slow} (x+y)\bigr), \\
\frac{dy}{dt} &= \delta xy - \gamma y 
                 + \varepsilon_2 y \, 
                   \bigl(\sin(\omega_{\rm fast} t) \cos(\omega_{\rm fast} xy) \bigr) \notag
             \times \sin\bigl(\omega_{\rm slow} (x+y)\bigr) 
                         \sin\Bigl(\frac{x}{y+c_{\rm reg}}\Bigr).
\end{align}

Here, $\varepsilon_1$ and $\varepsilon_2$ control the strength of missing dynamics not captured by the symbolic feature library, while $c_{\rm reg}$ is a small constant that prevents division by zero. 
The runs reported below use $(\varepsilon_1,\varepsilon_2)=(3.0,2.0)$ (see the per-system reproducibility values in \cref{app:exp_design}). 
The augmented terms introduce high-frequency temporal modulation, state-dependent coupling, and asymmetric predator-prey interactions; 
the explicit $\sin(\omega_{\rm fast} t)$ factors make the system non-autonomous (\cref{app:non_autonomy_caveat}). 

\paragraph{Results and Analysis.}

\begin{table}[h]
\centering
\small
\setlength{\tabcolsep}{5pt}
\caption{Lotka--Volterra results including derivative metrics.
Scale factors shown inside the metric parentheses indicate the power of ten applied to the reported values (e.g., $\mathrm{e}{-2}$).
Uncertainty terms reported as $\pm 0.0$ fall below the displayed precision after rescaling.
Abbreviations: MSE = mean squared error; ID = in-distribution; OOD = out-of-distribution.}
\label{tab:lv_results_streamlined}
\begin{tabular}{lccc}
\toprule
\textbf{Metric} & \textbf{Pure} & \textbf{$L^2$} & \textbf{OrthoReg} \\
\midrule
\multicolumn{4}{c}{\textbf{Predictive Performance}} \\
\midrule
MSE(${\scriptstyle \dot{x}_\mathrm{ID},\ \mathrm{e}{-2}}$) {\scalebox{0.8}{($\downarrow$)}} 
  & $1.64\ {\scalebox{0.8}{$\pm 0.00$}}$ 
  & $1.64\ {\scalebox{0.8}{$\pm 0.00$}}$ 
  & $1.64\ {\scalebox{0.8}{$\pm 0.00$}}$ \\

MSE(${\scriptstyle \dot{x}_\mathrm{OOD,T2},\ \mathrm{e}{-2}}$) {\scalebox{0.8}{($\downarrow$)}} 
  & $1.21\ {\scalebox{0.8}{$\pm 0.00$}}$ 
  & $1.22\ {\scalebox{0.8}{$\pm 0.00$}}$ 
  & $1.22\ {\scalebox{0.8}{$\pm 0.00$}}$ \\

MSE(${\scriptstyle \dot{x}_\mathrm{OOD,T3},\ \mathrm{e}{-1}}$) {\scalebox{0.8}{($\downarrow$)}} 
  & $1.74\ {\scalebox{0.8}{$\pm 0.00$}}$ 
  & $1.73\ {\scalebox{0.8}{$\pm 0.00$}}$ 
  & $\bm{1.71\ {\scalebox{0.8}{$\pm 0.00$}}}$ \\

\midrule
\multicolumn{4}{c}{\textbf{System Identification Quality}} \\
\midrule
F1 {\scalebox{0.8}{($\uparrow$)}} 
  & $0.22\ {\scalebox{0.8}{$\pm 0.01$}}$ 
  & $0.22\ {\scalebox{0.8}{$\pm 0.00$}}$ 
  & $\bm{0.24\ {\scalebox{0.8}{$\pm 0.01$}}}$ \\

\#Terms {\scalebox{0.8}{($\downarrow$)}} 
  & $16.6\ {\scalebox{0.8}{$\pm 0.9$}}$ 
  & $16.0\ {\scalebox{0.8}{$\pm 0.0$}}$ 
  & $\bm{14.8\ {\scalebox{0.8}{$\pm 0.4$}}}$ \\

\bottomrule
\end{tabular}
\end{table}

Across all predictive metrics in \cref{tab:lv_results_streamlined}, the three methods perform similarly in this configuration, indicating that the chosen library and training distribution already capture most of the dynamics relevant to our evaluation protocol.
OrthoReg yields a small but consistent gain in identification quality (F1: $0.24$ vs.\ $0.22$) while selecting slightly fewer terms.
We therefore view Lotka--Volterra as a ``low headroom'' setting: it serves as a sanity check that OrthoReg does not degrade performance when the decomposition problem is comparatively easy, rather than as a regime with large expected improvements.

\subsection{Duffing oscillator: multistability and cross-basin generalization}
\label{app:subsec_duffing}

This appendix complements the main-text results in \cref{tab:duffing_results_streamlined} by detailing the experimental setup and the origin of the cross-basin generalization challenge. Our Duffing configuration and the motivation for the basin split follow \citet{goring2024out}. \looseness=-1

\paragraph{System and regimes.}
We consider the unforced Duffing oscillator in first-order form,
\begin{align}
\dot{x} &= y, \\
\dot{y} &= ay - x(b + cx^2),
\end{align}
with parameters $[a,b,c]=[-\tfrac{1}{2},-1,\tfrac{1}{10}]$, placing the dynamics in a bistable regime with two symmetric basins of attraction at
$x=\pm\sqrt{-b/c}$ (and $y=0$). \citet{goring2024out} use this multistable setting to study out-of-domain generalization for learned dynamical-system models. \looseness=-1
\paragraph{Why cross-basin generalization is challenging.}
When training data are restricted to a single basin, the model never observes trajectories from the opposite basin. For multistable flows that are \emph{not topologically transitive} on the test domain, \citet{goring2024out} formalize that the resulting OOD generalization problem is not strictly learnable for universal-approximator hypothesis classes (Theorem~4.2), and empirically illustrate this behavior on the Duffing system. \looseness=-1

\paragraph{Library mismatch.}
The shared symbolic library (polynomial degree $2$ $\oplus$ Fourier($n=1$); \cref{app:exp_design}) does not contain the cubic term $cx^3$, so any method using this library must either (i) approximate the missing nonlinearity through other terms (pure symbolic) or (ii) represent it through the neural augmentation (hybrid models). \looseness=-1

\paragraph{Effect of orthogonal regularization.}
OrthoReg encourages the neural augmentation to be (approximately) orthogonal to the span of library terms on the training distribution, encouraging the symbolic component to retain the structure expressible by the library while pushing the neural component toward dynamics not captured by it. 
Empirically, this is associated with improved cross-basin performance and more stable identification compared to $L^2$-regularised hybrids (\cref{tab:duffing_results_streamlined}). \looseness=-1

\paragraph{Numerical note on basin crossings.}
Although the unforced system is basin-confining in the idealized continuous dynamics, numerical integration and initial conditions close to the separatrix can yield occasional apparent basin crossings. In the qualitative visualization we therefore treat trajectories as illustrative, and we generate/inspect trajectories to ensure the intended train-test basin split is respected (cf.\ \cref{app:exp_design} for integration details). \looseness=-1

\subsection{SIR: State Dependence and Memory}
\label{app:subsec_sir}

\paragraph{System Design.}

We extend the classical normalised SIR model with a multiplicative \emph{tweaked} term per compartment that combines a state-dependent bilinear coupling, an exponentially decaying time envelope, and a state-modulated frequency. Concretely, with $\omega_S=\omega_{\rm fast}(1+S)$, $\omega_I=\omega_{\rm fast}(1+2I)$, $\omega_R=\omega_{\rm slow}(1+R)$,

\begin{align}
\frac{dS}{dt} &= -\beta\,\frac{SI}{S+I+R}
                 \;+\; \varepsilon_1\, S\,I\, \sin(\omega_S t)\, e^{-t/\tau_{\rm memory}}\, \tanh\!\bigl(\omega_{\rm slow}(I+R)\bigr), \\
\frac{dI}{dt} &= +\beta\,\frac{SI}{S+I+R} - \gamma\, I
                 \;+\; \varepsilon_2\, I\,R\, \sin(\omega_I t)\, e^{-t/\tau_{\rm memory}}\, \tanh\!\bigl(\omega_{\rm slow}(S+R)\bigr), \\
\frac{dR}{dt} &= +\gamma\, I
                 \;+\; \varepsilon_3\, S\,R\, \sin(\omega_R t)\, e^{-t/\tau_{\rm memory}}\, \tanh\!\bigl(\omega_{\rm slow}(S+I)\bigr),
\end{align}

with parameters $\beta=2.0$, $\gamma=1.0$, $(\varepsilon_1,\varepsilon_2,\varepsilon_3)=(0.3,0.25,0.2)$, $\omega_{\rm fast}=3.0$, $\omega_{\rm slow}=1.5$, $\tau_{\rm memory}=2.0$ (\cref{app:exp_design}). Each tweaked term is added (sign $+$) to the corresponding standard SIR derivative, so $\varepsilon_i\to 0$ recovers the classical normalised SIR model exactly.
The base rates $\beta,\gamma$ are constant; all additional state dependence enters through the tweaked residual terms \citep{hethcote2000mathematics, kucharski2020early}. 
As in LV, the explicit time factors make this a non-autonomous stress test for the autonomous augmentation (\cref{app:non_autonomy_caveat}).

\begin{table}[h]
\centering
\small
\setlength{\tabcolsep}{5pt}
\caption{
SIR results under severe library mismatch. 
OrthoReg improves OOD-T3 extrapolation and symbolic separation, but trades off ID and OOD-T2 accuracy.}
\label{tab:sir_results_streamlined}
\begin{tabular}{lccc}
\toprule
\textbf{Metric} & \textbf{Pure} & \textbf{$L^2$} & \textbf{OrthoReg} \\
\midrule
\multicolumn{4}{c}{\textbf{Predictive Performance}} \\
\midrule
MSE(${\scriptstyle \dot{x}_\mathrm{ID},\ \mathrm{e}{-5}}$) {\scalebox{0.8}{($\downarrow$)}} 
  & $\bm{2.12 {\scalebox{0.8}{$\pm 0.73$}}}$ 
  & $31.42 {\scalebox{0.8}{$\pm 4.93$}}$ 
  & $90.64 {\scalebox{0.8}{$\pm 7.56$}}$ \\

MSE(${\scriptstyle \dot{x}_\mathrm{OOD,T2},\ \mathrm{e}{-4}}$) {\scalebox{0.8}{($\downarrow$)}} 
  & $\bm{2.40 {\scalebox{0.8}{$\pm 0.11$}}}$ 
  & $7.75 {\scalebox{0.8}{$\pm 0.53$}}$ 
  & $16.54 {\scalebox{0.8}{$\pm 0.80$}}$ \\

MSE(${\scriptstyle \dot{x}_\mathrm{OOD,T3},\ \mathrm{e}{-4}}$) {\scalebox{0.8}{($\downarrow$)}} 
  & $72.68 {\scalebox{0.8}{$\pm 19.04$}}$ 
  & $15.33 {\scalebox{0.8}{$\pm 3.33$}}$ 
  & $\bm{8.76 {\scalebox{0.8}{$\pm 0.52$}}}$ \\

\midrule
\multicolumn{4}{c}{\textbf{System Identification Quality}} \\
\midrule
F1 {\scalebox{0.8}{($\uparrow$)}} 
  & $0.14 {\scalebox{0.8}{$\pm 0.03$}}$ 
  & $0.15 {\scalebox{0.8}{$\pm 0.03$}}$ 
  & $\bm{0.20 {\scalebox{0.8}{$\pm 0.06$}}}$ \\

\#Terms {\scalebox{0.8}{($\downarrow$)}} 
  & $40.2 {\scalebox{0.8}{$\pm 3.1$}}$ 
  & $20.4 {\scalebox{0.8}{$\pm 3.6$}}$ 
  & $\bm{6.8 {\scalebox{0.8}{$\pm 3.3$}}}$ \\

Orth. {\scalebox{0.8}{($\uparrow$)}} 
  & --
  & $0.35 {\scalebox{0.8}{$\pm 0.16$}}$ 
  & $\bm{0.55 {\scalebox{0.8}{$\pm 0.23$}}}$ \\
\bottomrule
\end{tabular}
\end{table}

The SIR system demonstrates a challenging regime where the symbolic library has a severe mismatch with the true dynamics, a large number of candidate terms are present, and the tweaked term combines state-modulated frequencies $\sin(\omega_k(\cdot)t)$ with a decaying time envelope $e^{-t/\tau_{\rm memory}}$ that the autonomous augmentation can absorb only as a state-dependent envelope (\cref{app:non_autonomy_caveat}).
While all methods are challenged by this setup, OrthoReg retains an advantage in the out-of-distribution setting (OOD-T3) and in system identification (highest F1, fewest terms).

\paragraph{Summary.}
Taken together, the cross-system experiments show that OrthoReg remains most useful when the library captures part of the dynamics but leaves structured residuals. 
The gains are modest in low-headroom settings such as LV, clearer for cross-basin Duffing generalisation, and mixed under the strongly misspecified SIR variant, where OrthoReg improves OOD-T3 and sparsity but sacrifices ID and OOD-T2 accuracy.


\section{Experimental Design and Implementation}
\label{app:exp_design}

\paragraph{Overview.}

All experiments follow a common protocol: we generate synthetic trajectories from known ground-truth dynamical systems, train models to predict the vector field from state observations, and evaluate predictive accuracy, symbolic recovery, and neural--symbolic separation.
To isolate the effect of orthogonal regularization, OrthoReg and the hybrid baselines share the same model structure and training pipeline; they differ only in the form of the regularization applied to the neural augmentation.

\subsection{Evaluation metrics}

We report two complementary result metrics:
\begin{itemize}
\item \textbf{Predictive performance.} Mean-squared error (MSE) on derivatives, computed against the ground-truth vector field (available for all synthetic systems), and, when numerically stable, MSE of long-horizon trajectory rollouts in state space.
\item \textbf{System identification.} F1 score comparing the recovered symbolic support against the ground-truth library terms after coefficient thresholding at $10^{-3}$.
\end{itemize}
A useful diagnostic is the exploratory cosine $|\langle \hat f_{\mathrm{aug}}, \hat f_{\mathrm{phy}}\rangle_{\mathcal D}|/(\|\hat f_{\mathrm{aug}}\|_{\mathcal D}\|\hat f_{\mathrm{phy}}\|_{\mathcal D})$ between the learned augmentation and the assembled symbolic prediction $\hat f_{\mathrm{phy}}=\sum_j w_j \phi_j$. 
This quantity is an indirect proxy for the OrthoReg training penalty $\sum_j \langle \hat f_{\mathrm{aug}}, \phi_j\rangle_{\mathcal D}^2$ (the two coincide up to a library reweighting only when $w$ concentrates on a single term) and is numerically unstable when either component has small empirical norm, so we do not report it as a result and do not use it to rank methods.

\subsection{In-distribution and out-of-distribution evaluation}
\label{app:ood_protocol}

We distinguish four evaluation regimes:
\begin{itemize}
\item \textbf{In-distribution (ID):} test trajectories drawn from the same distribution of initial conditions and time points as training;
\item \textbf{In-distribution, time-extrapolated (ID,ext):} same initial-condition distribution as training, evaluated past the training horizon $T$ on the autonomous continuation of the trajectory; this stresses temporal extrapolation while keeping the IC distribution fixed;
\item \textbf{OOD-T2 (initial-condition extrapolation, ``IC-Ext'').} trajectories initialized outside the training range of initial conditions;
\item \textbf{OOD-T3 (parametric extrapolation, ``Param-Ext'').} trajectories generated from perturbed system parameters while keeping the symbolic library fixed, probing extrapolation under model mismatch.
\end{itemize}

Tables retain the OOD-T2/OOD-T3 column labels for compactness; 
the regime/basin axis corresponds to the Duffing cross-basin split (\cref{app:subsec_duffing}). 
Hyperparameters, including $\lambda$, are selected on a validation split disjoint from these test splits and fixed across systems.

\subsection{Data generation and derivative targets}
\label{sec:data_generation}

Trajectories are generated by numerically integrating the ground-truth vector field using \texttt{scipy.integrate.odeint} (LSODA adaptive solver) and then sampled to obtain state observations $\{x(t_i)\}$.
Models are trained using the vector-field regression objective (\cref{eq:fit_loss_vf}) with derivative targets $y_i \approx \dot x(t_i)$ computed via smoothed finite differences (PySINDy \texttt{SmoothedFiniteDifference}, second-order accuracy with Savitzky--Golay smoothing, window length $5$).

\paragraph{Temporal discretization.}
For the pendulum experiments we simulate trajectories over a horizon $T=6.0$ with $n_T=100$ time points, yielding a step size $\Delta t = T/(n_T-1) \approx 0.0606$.
Other systems use their respective simulation horizons and discretizations as specified by the experimental configuration.

\subsection{Model classes and symbolic library}

\paragraph{Neural augmentation.}
Hybrid models use a multilayer perceptron with three hidden layers and $\tanh$ activations.
The hidden width is $128$ for hybrid models in the main experiments. The pure-neural baselines (PINN and Universal ODE) reported in \cref{tab:baseline} are trained at the same hidden width $128$, so the neural-capacity comparison is at parity.
Weights are initialized using standard variance-preserving schemes as implemented in the respective models.

\paragraph{Symbolic library.}
We use a fixed feature library combining a polynomial basis (degree $2$, with cross-terms for multivariate systems) and a Fourier basis ($n_{\mathrm{frequencies}}=1$, i.e.\ $\sin$ and $\cos$ at the unit frequency). 
Throughout, $M$ counts scalar feature--component pairs (basis features times state dimension $n$), so for the pendulum ($n=2$) with poly-2 $\oplus$ Fourier-1 we have $M=20$, and for SIR ($n=3$) the same construction gives $M=45$.
Symbolic coefficients are optimized with an $L^1$ penalty and thresholded at $10^{-3}$ when computing support recovery metrics.

\subsection{Optimization and effective batch size}

All models are optimized using Adam.
The derivative regression phase is trained for $2000$ epochs with learning rate $0.0089$.
In the default Monte Carlo setup we use $n_{\mathrm{samples}}=1000$ trajectories, all of which fit into a single optimization batch; thus, each epoch effectively processes the full training set.
When minibatching is used, the batch size is stated explicitly.

\subsection{Baselines}
\label{app:baseline_implementation}
\paragraph{Pure symbolic (SINDy-like).}
Sparse regression over the symbolic library with an $L^1$ penalty, trained using the same optimizer infrastructure as the hybrid methods to ensure comparability. 
The canonical SINDy baseline uses sequential thresholded least squares (STLSQ).
\paragraph{$L^2$-regularized hybrid.}
Joint optimization over symbolic coefficients $w$ and neural parameters $\vartheta$ with loss

\begin{equation*}
\mathcal{L}
=
\mathcal{L}_{\mathrm{fit}}(w,\vartheta)
+\lambda_1\|w\|_1
+\lambda_2\|\hat f_{\mathrm{aug}}\|_{\mathcal{D}}^2.
\end{equation*}

\paragraph{OrthoReg.}
Identical to the $L^2$ hybrid baseline except that the neural regularizer is replaced by
\begin{equation*}
\lambda \sum_{j=1}^M
\langle \hat f_{\mathrm{aug}}, \phi_j \rangle_{\mathcal{D}}^2.
\end{equation*}

\begin{table}[h]
\centering
\small
\setlength{\tabcolsep}{2.5pt}
\caption{PySINDy STLSQ versus Pure (Adam $L^1$) and OrthoReg (\cref{tab:sweet_spot_results,tab:duffing_results_streamlined,tab:lv_results_streamlined,tab:sir_results_streamlined}). $\dagger$: divergent rollout or clamp.}
\label{tab:stlsq_comparison}
\begin{tabular}{l c c c}
\toprule
\textbf{Metric} & \textbf{STLSQ} & \textbf{Pure} & \textbf{OrthoReg} \\
\midrule
\multicolumn{4}{c}{\textbf{Pendulum} (\cref{eq:pend})} \\
\midrule
MSE(${\scriptstyle \dot{x}_\mathrm{ID},\ \mathrm{e}{-2}}$) & $\bm{1.17\ {\scalebox{0.8}{$\pm 0.00$}}}$ & $1.27\ {\scalebox{0.8}{$\pm 0.07$}}$ & $1.89\ {\scalebox{0.8}{$\pm 0.01$}}$ \\
MSE(${\scriptstyle x_\mathrm{OOD,T2},\ \mathrm{e}{0}}$) & $\dagger$ & $1.02\ {\scalebox{0.8}{$\pm 0.07$}}$ & $\bm{1.01\ {\scalebox{0.8}{$\pm 0.01$}}}$ \\
F1 & $0.70\ {\scalebox{0.8}{$\pm 0.00$}}$ & $0.43\ {\scalebox{0.8}{$\pm 0.09$}}$ & $\bm{0.93\ {\scalebox{0.8}{$\pm 0.15$}}}$ \\
\#Terms & $9.0\ {\scalebox{0.8}{$\pm 0.0$}}$ & $11.4\ {\scalebox{0.8}{$\pm 2.9$}}$ & $\bm{3.6\ {\scalebox{0.8}{$\pm 1.3$}}}$ \\
\midrule
\multicolumn{4}{c}{\textbf{Duffing} (\cref{app:subsec_duffing})} \\
\midrule
MSE(${\scriptstyle x_\mathrm{ID},\ \mathrm{e}{-1}}$) & $\dagger$ & $\bm{4.75\ {\scalebox{0.8}{$\pm 0.13$}}}$ & $5.59\ {\scalebox{0.8}{$\pm 0.09$}}$ \\
MSE(${\scriptstyle x_\mathrm{OOD,T2},\ \mathrm{e}{0}}$) & $\dagger$ & $11.03\ {\scalebox{0.8}{$\pm 6.09$}}$ & $\bm{4.93\ {\scalebox{0.8}{$\pm 0.35$}}}$ \\
F1 & $0.62\ {\scalebox{0.8}{$\pm 0.00$}}$ & $0.40\ {\scalebox{0.8}{$\pm 0.05$}}$ & $\bm{0.63\ {\scalebox{0.8}{$\pm 0.05$}}}$ \\
\#Terms & $7.0\ {\scalebox{0.8}{$\pm 0.0$}}$ & $11.0\ {\scalebox{0.8}{$\pm 1.6$}}$ & $\bm{3.4\ {\scalebox{0.8}{$\pm 0.5$}}}$ \\
\midrule
\multicolumn{4}{c}{\textbf{Lotka--Volterra} (\cref{app:subsec_lv})} \\
\midrule
MSE(${\scriptstyle \dot{x}_\mathrm{OOD,T2},\ \mathrm{e}{-2}}$) & $41.8\ {\scalebox{0.8}{$\pm 0.0$}}$ & $\bm{1.21\ {\scalebox{0.8}{$\pm 0.00$}}}$ & $\bm{1.22\ {\scalebox{0.8}{$\pm 0.00$}}}$ \\
F1 & $\bm{0.30\ {\scalebox{0.8}{$\pm 0.00$}}}$ & $0.22\ {\scalebox{0.8}{$\pm 0.01$}}$ & $0.24\ {\scalebox{0.8}{$\pm 0.01$}}$ \\
\midrule
\multicolumn{4}{c}{\textbf{SIR} (\cref{app:subsec_sir})} \\
\midrule
MSE(${\scriptstyle \dot{x}_\mathrm{OOD,T2},\ \mathrm{e}{-4}}$) & $233\ {\scalebox{0.8}{$\pm 0.00$}}$ & $\bm{2.40\ {\scalebox{0.8}{$\pm 0.11$}}}$ & $16.54\ {\scalebox{0.8}{$\pm 0.80$}}$ \\
F1 & $0.18\ {\scalebox{0.8}{$\pm 0.00$}}$ & $0.14\ {\scalebox{0.8}{$\pm 0.03$}}$ & $\bm{0.20\ {\scalebox{0.8}{$\pm 0.06$}}}$ \\
\bottomrule
\end{tabular}
\begin{flushleft}
\footnotesize
\end{flushleft}
\end{table}

\paragraph{PINN baseline.}
A physics-informed neural network that encodes known components of the pendulum dynamics as a soft constraint, while leaving unmodeled residual dynamics unconstrained.

\paragraph{Universal ODE.}
A fully neural model without a symbolic component; uses the same depth and activations as the hybrid networks.

\paragraph{Canonical PySINDy STLSQ (selection--refit baseline).}
\label{app:stlsq_baseline}

The Pure Symbolic rows in the main tables (\cref{tab:sweet_spot_results,tab:duffing_results_streamlined,tab:lv_results_streamlined,tab:sir_results_streamlined}) use the same continuous $L^1$+Adam sparse-regression pipeline as the hybrid models (\cref{app:baseline_implementation}), which gives a controlled comparison within one optimization stack.
As a canonical SINDy-style selection--refit reference, we additionally run PySINDy \texttt{STLSQ} on the same symbolic libraries and train/ID/OOD splits:
\texttt{SequentialThresholdedLeastSquares(threshold=0.045, normalize\_columns=False)} with second-order finite differences (\texttt{FiniteDifference(order=2)}).
For pendulum $\dot{x}_{\mathrm{OOD,T3}}$ under STLSQ, we fix the support selected on train and refit the active coefficients by ordinary least squares on the OOD-T3 trajectories before scoring.
Where trajectory rollouts diverged or hit the numerical clamp, state MSE is marked $\dagger$; STLSQ standard deviations are omitted when identical across all five seeds (deterministic splits and bitwise-identical fits).

The comparison underscores why OrthoReg targets partial library misspecification rather than purely symbolic reconstruction. 
STLSQ reduces shrinkage bias and improves support recovery relative to the continuous \(L^1\)+Adam Pure baseline on the pendulum and Duffing systems. 
However, because STLSQ has no residual component, out-of-library dynamics must be approximated within the fixed dictionary, which can lead to unstable rollouts even when derivative-fit or F1 diagnostics look reasonable. 
OrthoReg instead couples sparse symbolic recovery with a flexible residual discouraged from using library directions, yielding stronger pendulum recovery and finite hybrid rollouts in the regimes above.
The LV and SIR results show the complementary limitation: when the library mismatch is severe or the evaluation is dominated by time-dependent perturbations, F1 alone is not a reliable summary of dynamical performance.

\subsection{Per-system reproducibility values}
\label{app:repro_values}

For each system we list (i) the ground-truth parameters used for training trajectories, (ii) the perturbation used to generate the OOD-T3 split, and (iii) the simulation horizon and step size. 
Final regularization hyperparameters are shared across systems and listed at the end of the block; 
the same values are used for both the $L^2$ baseline and OrthoReg, applied to the corresponding regularization term.

\begin{itemize}[leftmargin=*]
\item \textbf{Modified damped pendulum} (\cref{eq:pend}, dataset \texttt{theoretical\_pendulum.yaml}):
  \begin{itemize}
  \item parameters $\omega_0=1.0$, $\alpha=0.2$, $(\beta_1,\beta_2,\beta_3)=(0.3,0.25,0.15)$;
  \item OOD-T3 perturbation: all parameters scaled by $1.2$;
  \item $T=6.0$, $n_T=100$, $\Delta t\approx 0.0606$.
  \end{itemize}
\item \textbf{Duffing oscillator} (\cref{app:subsec_duffing}, dataset \texttt{duffing.yaml}):
  \begin{itemize}
  \item parameters $a=-0.5$, $b=-1.0$, $c=0.1$ (cubic term $cx^3$ absent from the dictionary);
  \item OOD-T3 perturbation: $(a,b,c)\mapsto(1.2\,a,\,1.2\,b,\,2.0\,c)$;
  \item $T=40$, $\Delta t=0.01$.
  \end{itemize}
\item \textbf{Lotka--Volterra} (\cref{app:subsec_lv}, dataset \texttt{complex\_orthogonal\_lv.yaml}):
  \begin{itemize}
  \item base parameters $\alpha,\beta,\gamma,\delta$ as in \texttt{lv.py} (Blasius-style scaled rates $0.1\!\cdot\!12,\,0.005\!\cdot\!12,\,0.04\!\cdot\!12,\,0.00004\!\cdot\!12$); missing-dynamics weights $(\varepsilon_1,\varepsilon_2)=(3.0,2.0)$ at the reported difficulty $2.0$;
  \item OOD-T3 perturbation: parameter ranges $(0.2,0.3)\!\cdot\!12$, $(0.01,0.015)\!\cdot\!12$, $(0.08,0.12)\!\cdot\!12$, $(0.00008,0.00012)\!\cdot\!12$ for $(\alpha,\beta,\gamma,\delta)$ respectively, with initial conditions $(N_0,P_0)\in(100,200)\times(10,20)$;
  \item $T=10$, $\Delta t=0.1$.
  \end{itemize}
\item \textbf{Tweaked SIR} (\cref{app:subsec_sir}, dataset \texttt{tweaked\_sir.yaml}):
  \begin{itemize}
  \item base parameters $\beta=2.0$, $\gamma=1.0$; tweaked-term parameters $(\varepsilon_1,\varepsilon_2,\varepsilon_3)=(0.3,0.25,0.2)$, $\omega_{\rm fast}=3.0$, $\omega_{\rm slow}=1.5$, $\tau_{\rm memory}=2.0$;
  \item OOD-T3 perturbation: $\beta\in(8,12)$, $\gamma\in(0.8,1.2)$ (parameter rescaling applied to the standard SIR backbone);
  \item $T=10$, $\Delta t=0.1$.
  \end{itemize}
\item \textbf{Final regularization hyperparameters} (selected per \cref{app:hyperparam_selection} and shared across systems):
  \begin{itemize}
  \item $L^2$ baseline: $\lambda_2=0.005$ (on $\|\hat f_{\rm aug}\|_{\mathcal D}^2$), symbolic sparsity $\mu=0.001$;
  \item OrthoReg: $\lambda=0.005$ (on the inner-product penalty), symbolic sparsity $\mu=0.003$.
  \end{itemize}
\end{itemize}

\subsection{Regularization selection}
\label{app:hyperparam_selection}

We select regularisation strengths $(\lambda,\mu)$ for each method by a small validation sweep, using prediction loss on a held-out in-distribution trajectory split as the selection criterion. 
Symbolic recovery is tracked over the same sweep but reported only as a diagnostic. 
The selected $(\lambda,\mu)$ are fixed across in-distribution and out-of-distribution evaluations; the same protocol is applied to the $L^2$ baseline with its own grid.

\paragraph{Penalty scale dependence on library norms.}
The orthogonality penalty $\lambda\sum_j\langle\hat f_{\rm aug},\phi_j\rangle_{\mathcal D}^2$ is scale-dependent in the library: 
rescaling $\phi_j\to c_j\phi_j$ rescales the per-feature contribution by $c_j^2$. 
The reported runs operate on bounded state ranges where the empirical norms $\|\phi_j\|_{\mathcal D}$ are comparable across $j$; 
for libraries with very different per-feature scales, standardising features to unit empirical norm (and rescaling $\lambda$ accordingly) is a robust drop-in.

\section{LLM Usage Disclosure}

Large language models were used as writing and programming assistants throughout the project, including for text polishing, code development, and improving the presentation of mathematical arguments. 
The research idea, theoretical results, experimental design, analyses, figures, and manuscript were produced by the authors.


\section{Limitations}
\label{app:limitations}

\paragraph{Scope of the guarantee.}
Our guarantees apply to the additive decomposition $f=f_{\mathrm{phy}}+f_{\mathrm{aug}}$ under the empirical inner product used for training. 
They therefore do not directly cover compositional architectures such as $\hat f=T_w+N_\vartheta\circ T_w$, unconstrained residual-learning pipelines, or population correlations of the learned augmentation without an additional uniform-convergence argument (\cref{prop:seq_two_stage_failure,app:sequential_extension}). 
A full population-level theory is left for future work.

\paragraph{Empirical-to-population transfer.}
The penalty bound (\cref{thm:penalty_bound_main}) and the error decomposition (\cref{thm:error_decomp}) are stated for the empirical inner product $\langle\cdot,\cdot\rangle_{\mathcal D}$. 
A bound on the population correlation $\langle\hat f_{\mathrm{aug}},\phi_j\rangle_\nu$ for the \emph{learned} augmentation cannot be obtained by single-function concentration alone (the function depends on $\mathcal D$); 
it would require uniform convergence over the augmentation hypothesis class (e.g.\ Rademacher complexity or VC-style arguments). 
A uniform-convergence statement is left to future work.

\paragraph{Symbolic-discovery regime.}
OrthoReg is formulated for a fixed, finite, differentiable library and a continuous sparsity relaxation optimised jointly with the neural augmentation. 
Selection--refit methods such as STLSQ, $L^0$ search, genetic programming, and transformer-based symbolic search operate in a different regime: 
after support selection, coefficients are typically refit by unregularised least squares, avoiding shrinkage bias on the retained terms. 
Combining such solvers with orthogonality-regularised residual learning is a natural extension, but it is not covered by the joint-objective certificate analysed here.

\paragraph{Dictionary-dependent interpretability and OOD scope.}
The recovered symbolic support is mechanistically meaningful only when the chosen library spans a scientifically relevant component of the true vector field. 
\Cref{thm:error_decomp} separates in-library and orthogonal-complement error under the empirical inner product; 
consequently, OrthoReg is most useful when missing dynamics are well separated from the library span, and less beneficial when the residual is strongly correlated with library functions on the observed distribution. 
OrthoReg therefore provides a complementary empirical decomposition, not recovery of governing laws absent from the library.

\paragraph{Derivative supervision and noise.}
The experiments use vector-field regression with derivative targets estimated from smoothed finite differences. 
This isolates the decomposition problem but makes high-noise settings sensitive to derivative-estimation error, as seen in \cref{app:noise_robustness}; 
in particular, at $\varepsilon\gtrsim 0.1$ the smoothed-FD targets (Savitzky--Golay window of length~5) themselves limit recoverability. 
State-prediction losses (\cref{eq:fit_loss_step}) or noise-aware differentiation are natural extensions for sensor-level data.

\paragraph{Implementation.}
The reported implementation optimises the continuous $L^1$-regularised objective with Adam and obtains exact sparsity by post-hoc thresholding. 
The orthogonality penalty is scale-dependent in the library features, so feature standardisation is advisable for libraries with very different per-feature scales (\cref{app:hyperparam_selection}). 
The Pure-Symbolic rows use the same continuous sparsity pipeline as the hybrid models for controlled comparison; 
STLSQ-style pure-symbolic baselines are discussed in \cref{app:baseline_implementation}.


\end{document}